\documentclass[11pt]{article}
\usepackage[left=1in,right=1in,top=1in,bottom=1in]{geometry}
\usepackage[numbers,sort&compress]{natbib}
\usepackage[utf8]{inputenc} % allow utf-8 input
\usepackage[T1]{fontenc}    % use 8-bit T1 fonts
\usepackage{url}
\usepackage{hyperref}            % simple URL typesetting
\usepackage{booktabs}       % professional-quality tables
\usepackage{amsfonts}       % blackboard math symbols
\usepackage{nicefrac}       % compact symbols for 1/2, etc.
\usepackage{microtype}
\usepackage{mathtools}
\usepackage{amsmath}
\usepackage{amsthm}
\usepackage{multirow}
\usepackage{algpseudocode}
\usepackage{amssymb}
\usepackage{graphicx}
\usepackage{subcaption}
\usepackage{tikz}
\usepackage{adjustbox}
\usetikzlibrary{calc}
\usepackage{pgfplots}
\usepackage{filecontents}
\usepackage{pgfplotstable, booktabs}
\usepgfplotslibrary{groupplots}
\usetikzlibrary{patterns}

\makeatletter
\def\pgfplots@install@path@replacements{%
    \ifpgfplots@path@replace@ellipse
        \let\tikz@do@circle=\pgfplots@path@@tikz@do@circle
        \let\tikz@do@ellipse=\pgfplots@path@@tikz@do@circle
        \expandafter\def\expandafter\pgfinterruptpicture\expandafter{\pgfinterruptpicture
            \let\tikz@do@circle=\pgfplots@path@@tikz@do@circle@orig
            \let\tikz@do@ellipse=\pgfplots@path@@tikz@do@ellipse@orig
        }%
    \fi
}%
\let\pgfplots@path@@tikz@do@circle@orig=\tikz@do@circle
\let\pgfplots@path@@tikz@do@ellipse@orig=\tikz@do@ellipse

\let\pgfplots@path@@tikz@do@circle@oldandbroken=\pgfplots@path@@tikz@do@circle
\def\pgfplots@path@@tikz@do@circle#1{\pgfplots@path@@tikz@do@circle@oldandbroken{#1}{#1}}
\def\pgfplots@path@@tikz@do@ellipse#1#2{\pgfplots@path@@tikz@do@circle@oldandbroken{#1}{#2}}
\makeatother

\pgfplotsset{
  ylabel right/.style={
    after end axis/.append code={
      \node [rotate=90, anchor=north] at (rel axis cs:1,0.5) {#1};
    }   
  }
}

\usepackage{graphicx}
\usepackage[capitalize]{cleveref}
\usepackage{caption}
\usepackage{subcaption}
\usepackage{tabularx}
\usepackage{color}
\usepackage{colortbl}
\usepackage{algorithm}
\usepackage{algpseudocode}
\usepackage[normalem]{ulem}
\usepackage{float}
\usepackage[section]{placeins}

\makeatletter
\def\thickhline{%
  \noalign{\ifnum0=`}\fi\hrule \@height \thickarrayrulewidth \futurelet
   \reserved@a\@xthickhline}
\def\@xthickhline{\ifx\reserved@a\thickhline
               \vskip\doublerulesep
               \vskip-\thickarrayrulewidth
             \fi
      \ifnum0=`{\fi}}
\makeatother

\newlength{\thickarrayrulewidth}
\definecolor{grey}{rgb}{0.7,0.7,0.7}
\definecolor{bgreen}{rgb}{0.7,1,0.7}
\definecolor{green}{RGB}{4,101,53}
\definecolor{bgreen}{RGB}{98,159,67}
\definecolor{dbrown}{RGB}{75,54,33}
\definecolor{bblue}{RGB}{10,157,217}
\definecolor{blue}{RGB}{39,59,129}
\definecolor{bred}{RGB}{217,10,100}
\definecolor{red}{RGB}{160,30,50}

\newcommand\numberthis{\addtocounter{equation}{1}\tag{\theequation}}

\DeclareMathOperator{\E}{\mathrm{E}}
\DeclareMathOperator{\Var}{\mathrm{Var}}

\newtheorem{theorem}{Theorem}
\newtheorem{lemma}{Lemma}

\algnewcommand{\LineComment}[1]{\State \(\triangleright\) {\tt #1}}

\makeatletter
\def\BState{\State\hskip-\ALG@thistlm}
\makeatother

\newcommand{\pgmpool}{\langle G, \hat{\theta} \rangle}
\newcommand{\pgmpop}{\langle G, \theta \rangle}
\newcommand{\pgmpopp}{\langle G', \theta' \rangle}
\newcommand{\parpool}[1][i]{Pa_{X_{#1}}^{G};\hat{\theta}}
\newcommand{\pgmdifpop}{\langle G_{pop}, \theta \rangle}

\newcommand{\hin}{\text{H}_{\text{IN}}}
\newcommand{\hout}{\text{H}_{\text{OUT}}}

\newcommand{\parpop}[1][i]{Pa_{X_{#1}}^{G};\theta}

\newcommand{\ind}[2]{1_{\{Pa_{X_{#1}}^G = {#2}\}}}

\title{Quantifying the Privacy Risks of Learning High-Dimensional Graphical Models}

\author{
  Sasi Kumar Murakonda, Reza Shokri, George Theodorakopoulos$^\dagger$  \\
  National University of Singapore (NUS), $^\dagger$Cardiff University\\
  \texttt{\{murakond, reza\}@comp.nus.edu.sg, TheodorakopoulosG@cardiff.ac.uk}
}

\date{}

\begin{document}

\maketitle

\begin{abstract}
  Models leak information about their training data.  This enables attackers to infer sensitive information about their training sets, notably determine if a data sample was part of the model's training set.  The existing works {\em empirically} show the {\em possibility} of these  membership inference (tracing) attacks against complex deep learning models. However, the attack results are dependent on the specific training data, can be obtained {\em only after} the tedious process of training the model and performing the attack, and are missing any measure of the confidence and unused potential power of the attack.  

  In this paper, we {\em theoretically} analyze the maximum power of tracing attacks against high-dimensional graphical models, with the focus on Bayesian networks.  We provide a tight upper bound on the power (true positive rate) of these attacks, with respect to their error (false positive rate), for a given model structure {\em even before} learning its parameters. As it should be, the bound is independent of the knowledge and algorithm of any specific attack. It can help in identifying which model structures leak more information, how adding new parameters to the model increases its privacy risk, and what can be gained by adding new data points to decrease the overall information leakage. It provides a measure of the potential leakage of a model given its structure, as a function of the model complexity and the size of the training set.
\end{abstract}

% !TEX root = ../main.tex

\section{Introduction}

How much is the privacy risk of releasing high-dimensional models which are trained on sensitive data?  We focus on measuring information leakage of models about their training data, using tracing (membership inference) attacks. In a tracing attack, given the released model and a target data sample, the adversary aims at inferring whether or not the target sample was a member of the training set. We use the term tracing attack and membership inference attack interchangeably~\cite{dwork2017exposed, shokri2017membership}. The attack is evaluated based on its power (true positive rate), and its error (false positive rate), in its binary decisional task. 

Tracing attacks have been extensively studied for summary statistics, where independent statistics (e.g., mean) of attributes of high-dimensional data are released. Homer et al.~\cite{homer2008resolving} showed the existence of powerful tracing attacks; more recent work provided theoretical frameworks to analyze the upper bound on the power of these inference attacks~\cite{sankararaman2009genomic}, and their robustness to noisy statistics~\cite{dwork2015robust}.  The theoretical analysis helps explain the major causes of information leakage and assess the privacy risk even before computing the statistics.  However, the analysis in~\cite{homer2008resolving,sankararaman2009genomic,dwork2015robust} is limited to simple models such as product distributions. 

Advanced machine learning models, such as deep neural networks, have recently been tested against tracing attacks.  In the black-box setting, the attacker can only observe predictions of the model.  The attack involves training inference models that can distinguish between members and non-members from the predictions that the target model produces~\cite{shokri2017membership}.  The attacks are tested against deep neural networks as well as machine-learning-as-a-service platforms, and their accuracy is shown to be related to their generalization error~\cite{shokri2017membership, yeom2018privacy}.  In the white-box setting, the attacker obtains the parameters of the model, and decides that the target sample is a member if the gradients of the loss with respect to the model's parameters computed on the target data sample are aligned with the model's parameters~\cite{nasr2019comprehensive}. Large models are empirically shown to be more vulnerable to the attack, even if they have better generalization performance.  These attacks highlight the susceptibility of high-dimensional neural networks to tracing attacks.  However, their analysis is limited to empirical measurements of the attack success, after training the models on particular data sets. 

\paragraph*{\bf\em Contributions} Using the above-mentioned existing methods, it is possible to reason theoretically about tracing attacks, yet only for simple models (independent statistics). In parallel, it is possible to perform empirical tracing attacks against complex models (deep neural networks), yet without much theoretical analysis on the maximum power of inference attacks.  In this paper, we aim at addressing this gap by providing a theoretical analysis of tracing attacks against high dimensional graphical models, i.e. models with many parameters. We provide a bound on the performance of tracing attacks to quantify the privacy risk that learning a model implies for the training set. Using the bound, one can determine the elements of a released model that contribute to the power of the attacker.  Our focus is on {\em probabilistic graphical models}, which are very general statistical models that capture the correlations among data attributes, as they are among fundamental models for machine learning, and the basis of deep learning models via e.g., restricted Boltzmann machines.

We use the likelihood ratio test (LR test) as the foundation of our tracing attack~\cite{sankararaman2009genomic}.  This enables us to {\bf design the most powerful attack} against any probabilistic model.  Thus, for any given error, there exists no other attack strategy that can achieve a higher power.  Our objective is not to empirically evaluate the performance of attacks (even the theoretically strongest one) on trained models.  We instead {\bf compute the maximum achievable power of tracing attacks}.  This upper bound can be used as a measure to evaluate the effectiveness of different attack algorithms by comparing their achieved power to the bound for any false positive error.  

Our objective is to identify the elements of a model that cause membership information leakage and measure their influence.  We prove that, for a given model structure, the potential leakage of the model (the leakage that corresponds to the most powerful attack for any given error) is proportional to the square root of model's complexity (defined as the number of its independent parameters), and is inversely proportional to the square root of the size of the training set.  Thus, the theoretical bound enables us to {\bf quantify the potential leakage of a model before even learning the parameters of the model} on that structure. This can be used to efficiently compare different model structures based on their susceptibility to tracing attacks. The theoretical bound can quantify the power that the attack gains/loses if a new attribute is added/removed from the data, or when the model requires capturing/removing the correlation between certain attributes.  It also determines the size of the training set for a very high-dimensional model that leaks a similar amount of information as a small model leaks on a small set of data. 

We evaluate our attack against real (sensitive) data: location check-ins, purchase history, and genome data.  We empirically show that the upper bound is tight, and the power of the likelihood ratio test attacker is extremely close to the bound for any false positive error. 
% !TEX root = ../main.tex

\section{Probabilistic Graphical Models}\label{sec:pgm}

Probabilistic graphical models make use of a graph-based representation to encode the dependencies and conditional independence between random variables~\cite{koller2009probabilistic}.  Each node $X_i$ in a graph $G$ is a random variable, and the edges represent the dependencies.  In this paper, we focus on {\bf Bayesian networks}.  However, our attack framework is easily applicable to Markov random fields.  In Bayesian networks, the model structure is a directed acyclic graph, and the model $\pgmpop$ enables factoring the joint probability over the random variables.  Given all its parents in the graph, a random variable is conditionally independent of other random variables.  Therefore, the joint distribution can be factored as follows. 
\begin{align}\label{eq:factoring}
	\Pr[X_1, X_2, \cdots, X_m] = \prod_{i=1}^{m} \Pr[X_i | \parpop],
\end{align}
where $Pa_{X_{i}}^{G}$ is the parent random variables of $X_i$.  The parameters $\theta$ of the model encode the conditional probabilities. 

We define the {\bf complexity} $C(G)$ of a Bayesian network $\pgmpop$ with discrete random variables as the number of independent parameters used to define its probability distribution.   Let $V(X)$ be the number of distinct values that a random variable $X$ can take.  For each conditional probability $\Pr[X_i | \parpop]$, we need $|V(Pa_{X_i}^G)| (|V(X_i)| - 1)$ independent parameters. Thus, the total complexity of a model is: 
\begin{align}\label{eq:complexity}
	C(G) = \sum_{i=1}^m |V(Pa_{X_i}^G)| (|V(X_i)| - 1).
\end{align}

\section{Problem Statement}\label{sec:problem}

We consider a set of $n$ independent $m$-dimensional data samples from a {\em population}.  We refer to this set as the {\em pool}. We do not make any assumption about the probability distribution of the data points in the general population. Given a graphical model structure $G$, the pool data is used to train a graphical model, i.e., to estimate the parameters $\hat{\theta}$ of the probabilistic graphical {\bf model} $\pgmpool$. The estimation can be either Max likelihood estimation (MLE) or Max a posteriori (MAP) estimate. This model is {\em released}.  Our objective is to quantify the privacy risks of releasing such models for the members of their training data. 

Let us consider an adversary who observes the released model $\pgmpool$. We assume that the attacker can collect a set of independent samples from the population.  We refer to this set as the {\em reference population}.  The objective of the adversary is to perform a {\bf tracing attack} (also known as the membership inference attack) against the released model, on any target data point $x$: create a decision rule that determines whether $x$ was used in the training of the parameters of $\pgmpool$ or not, i.e. to classify $x$ as being in the pool (IN) or not (OUT).

The accuracy of the tracing attack indicates the information leakage of the model about the members of its training set.  We quantify the attacker's success using two evaluation metrics: the adversary's {\bf power} (the true positive rate), and his {\bf error} (the false positive rate).  The power measures the conditional probability that the attacker classifies $x$ as IN, given that $x$ is indeed in the pool, i.e. $\Pr[IN|x \in \text{pool}]$.  The error measures the conditional probability that the attacker classifies $x$ as IN, given that $x$ is not in the pool, i.e. $\Pr[IN|x \notin \text{pool}]$. The ROC curve (Receiver Operating Characteristic), which is a plot of power versus error, captures the trade-off between power and error.  Thus, the area under the ROC curve (AUC) is a single metric for measuring the strength of the attack. The AUC can be interpreted as the probability that a randomly drawn data point from the pool will be assigned larger probability of being IN than a data point randomly drawn from outside the pool. So AUC = 1 implies that the attacker can correctly classify all data samples as IN or OUT.

% \subsection{Framework for Attack Design}

% To design the most powerful membership inference attack and quantify its power, we first need to understand the cause of information leakage from models. Consider a random vector $X$ following a probability distribution $P$ defined over a discrete space.  For simplicity assume $X$ is one-dimensional. Let $\mu$ be the value of some statistic $\theta$ about $X$ (e.g. the mean) calculated using $P$. The pool dataset $D$ is created by sampling $n$ i.i.d. data points from $P$. Let the value of the mean estimated from $D$ be $\mu_D$. 
% The deviation of $\mu_D$ from $\mu$ reveals information about the samples present in $D$. This deviation -- in general, the {parameter's estimation error} for the pool dataset $D$ -- helps the adversary to distinguish the members of $D$ from any random data point generated by $P$. 

% The key to performing membership inference attack on $D$ given $\mu_D$ is to exploit its deviation from $\mu$. To do so, the attacker needs to know/estimate $\mu$.  Thus, the first step is to compute an accurate estimation of the parameter on the true distribution $P$, for example, using a large number of samples from $P$.  We then need to test the influence of a data point on $\mu_D$. To do so, we model the decisional problem of membership inference as a hypothesis test, as in the prior work~\cite{homer2008resolving, sankararaman2009genomic}. 

\subsection{Membership Inference Attack against Graphical Models}\label{sec:attack}

Given the released model, the reference population, and the target data point, the adversary aims at distinguishing between two hypothesis.  Each hypothesis describes a possible world that could have resulted in the observation of the adversary, where in one world the target data was part of the training set (pool), while in the other one the target data is a random sample from the population.

\begin{itemize}
	\item Null hypothesis ($\hout$): The pool is constructed by drawing $n$ independent samples from the general population.  Parameters $\hat{\theta}$ of the model $\pgmpool$ are learned on the pool data.  Target data $x$ is drawn from the general population, independently from the pool.
	\item Alternative hypothesis ($\hin$): The pool is constructed by drawing $n$ independent samples from the general population.  Parameters $\hat{\theta}$ of the model $\pgmpool$ are learned on the pool data. Target data $x$ is drawn from the pool.
\end{itemize}

This generalizes the hypothesis test designed by Sankararaman et al.~\cite{sankararaman2009genomic}, in which the released model follows a product distribution, i.e. the random variables corresponding to data attributes are independent (equivalent to a probabilistic graphical model without any dependency edges between the nodes).

We use the Likelihood Ratio test to distinguish the two hypotheses. The goal of hypothesis testing is to find whether there is \textbf{enough evidence to reject the null hypothesis \textit{in favor of the alternative hypothesis},} i.e. whether the likelihood $L_\text{IN}$ of the alternative hypothesis is large enough compared to the likelihood $L_\text{OUT}$ of the null hypothesis. The only information we know about the pool is $\hat{\theta}$, the parameters of the released model learned using the pool data.  Hence, we must calculate these {\em exact same parameters} under null hypothesis (i.e., learn the parameters using general population).  Let $\theta$ be the result of this computation, i.e., the parameters of $G$ trained on a large reference population. We calculate $L_\text{IN}$ as the likelihood of the parameters of $G$ taking the value $\hat{\theta}$, which is equal to $\Pr[x ; \pgmpool]$. Similarly, we calculate $L_\text{OUT}$ as the likelihood of the parameters of $G$ taking the value $\theta$, which is equal to $\Pr[x ; \pgmpop]$. 

Hence, the log likelihood statistic is computed as follows. 
\begin{equation}\label{eq:log-likelihood-ratio}
L(x) = \log\left(\frac{\Pr[x ; \pgmpop]} {\Pr[x ; \pgmpool]}\right)
\end{equation}

The LR test is a comparison of the log likelihood statistic $L(x)$ with a threshold.  If $L(x) \leq \text{threshold}$, then the attacker decides in favor of $\hin$ (rejects $\hout$); else, he decides in favor of $\hout$ (more precisely, he fails to reject $\hout$ because there is not enough evidence to support this rejection in favor of $\hin$). To determine the threshold, the attacker selects a (false positive rate) error $\alpha$ that he is willing to tolerate.  He then empirically or theoretically estimates the distribution of $L(x)$ under the null hypothesis, using his reference population. We denote the CDF of this distribution as $F$.  Given $\alpha$ and $F$, the attacker computes a threshold value $F^{-1}(\alpha)$ and compares it to $L(x)$, to decide whether to reject the null hypothesis. This concludes the hypothesis test.

The \emph{power} of the test, as defined earlier, can be expressed as $\Pr[L(x) \leq F^{-1}(\alpha)]$, computed under the alternative hypothesis, for an individual data point $x$ randomly drawn from the pool. In other words, it is the fraction of pool data points that are correctly classified by the test.  By varying $\alpha$, and thus the threshold $F^{-1}(\alpha)$, we can draw the ROC curve and compute the AUC metric as well.  It is worth emphasizing that according to the Neyman-Pearson lemma~\cite{neyman1933ix}, the LR test achieves the {\bf maximum power} among all decision rules with a given error (false positive rate).  So, any other decision rule would result in a lower AUC. 
% !TEX root = ../main.tex

\section{Theoretical analysis - Bound on power of attack}\label{sec:attack-analysis}

Our objective is to compute the maximum power $\beta$ for any false positive error $\alpha$ of an adversary that observes the released model $\pgmpool$ which has been trained on a pool of size $n$. In our main result, Theorem~\ref{thm:main-result}, we show which combinations of $\alpha$ and $\beta$ are possible for the attacker, and we find the major factors that determine these combinations, as a function of the model complexity and size of the dataset. To derive our main result about the best achievable power-error tradeoff, we assume that the released parameters satisfy the below conditions.

\begin{itemize}
	\item The value of every released parameter is learned from a large enough number of samples for the central limit theorem to hold good. 
	\item The value of every released parameter is non-trivial i.e., it is bounded away from 0 and 1~\cite{sankararaman2009genomic}.
\end{itemize}

These are valid assumptions to make on part of the model publisher, as it is not beneficial to publish statistically insignificant or trivial estimates. In fact, the recently published methodology of learning Bayesian Networks on Cancer Analysis System (CAS) database in the National Cancer Registration and Analysis Service (NCRAS) has similar assumptions (they use only the parameters that are learned using at least 50 samples)\footnote{\url{https://simulacrum.healthdatainsight.org.uk/publications}}

%~\cite{Simulacrum}. 

\begin{theorem}\label{thm:main-result}
	Let $\beta$ and $\alpha$ be the power and error of the LR test, for the membership inference attack, respectively.  Let $n$ be the size of the pool (model's training set), and $C(G)$ be the complexity of the released probabilistic graphical model $\pgmpool$. Then, the tradeoff between power and error follows the following relation: 
	\begin{equation}
	z_{\alpha} + z_{1-\beta} \approx \sqrt{\frac{C(G)}{n}},
	\end{equation}
	where $z_s$ is the quantile at level $1 - s, 0 < s < 1$ of the Standard Normal distribution.
\end{theorem}

\begin{proof}[Proof sketch]
To compute $\beta = \Pr_{pool}\{L(x) \leq F^{-1}(\alpha)\}$, the power of the LR test for the inference attack, for any error $\alpha$, we need the distribution of $L(x)$ when $x$ is drawn from the pool and when $x$ is drawn from the population. Our approach to estimating the distributions of $L(x)$ is through computing its moments $\E(L^k), k > 0$.  To approximate the distribution using its moments, we use an established statistical principle for fitting a distribution with known moments: the maximum-entropy principle. This principle states that the probability distribution which best represents the current state of knowledge is the one with largest entropy \cite{jaynes1957information1,jaynes1957information2}.  

To simplify the computation of the moments, we take advantage of the Bayesian decomposition to split this $L(x)$ as sum of simpler terms (one for each attribute $X_i$).  We start by expanding \eqref{eq:log-likelihood-ratio} to give the following expression for $L(x)$:
\begin{align*}
L(x) =& \log\left(\frac{\Pr[x ; \pgmpop]} {\Pr[x ; \pgmpool]}\right)
= \log\left(\frac{\prod_{i=1}^{m} \Pr[X_i | \parpop]}{\prod_{i=1}^{m} \Pr[X_i | \parpool]}\right)\\
=& \sum_{i=1}^{m} \log\left(\frac{\Pr[X_i | \parpop]}{\Pr[X_i | \parpool]}\right) \numberthis
\end{align*}
where the $X_i$ are the attributes of the data point $x$, which is now a random variable as it is drawn from the pool (or population), as just mentioned. We define $L_i$ as the contribution of attribute $X_i$ to the likelihood ratio $L$. Hence the value of $L_i$ can be calculated as:
\begin{align}
L_i = \log\left(\frac{\Pr[X_i | \parpop]}{\Pr[X_i | \parpool]}\right) \label{eq:Li-definition}
\end{align}

We calculate the first two moments of $L(x)$ for our approximation. The mean and variance of $L(x)$ are $\mu_0 = \frac{C(G)}{2n}, \sigma_0^2 = \frac{C(G)}{n}$ under the null hypothesis and $\mu_1 = -\frac{C(G)}{2n}, \sigma_1^2 = \frac{C(G)}{n}$ under the alternative hypothesis (See proof in Appendix~\ref{subsec:derivation-mean-variance}). For a known mean $\mu$ and variance $\sigma^2$, the max-entropy distribution that matches the target distribution is a Gaussian $N(\mu, \sigma^2)$. Deriving higher order moments of $L(x)$ requires information about the exact distribution that generated the data. Note that in our analysis, we do not make any assumption on the distribution from which data is generated. We just assume that the pool and the population are from the same distribution. Making assumptions on the distribution that generated the data limits the practical utility of such bounds (as we want to estimate potential leakage from a model, before touching the data). See Appendix~\ref{subsec:derivation-mean-variance} for details on how the data generator distribution affects the distribution of the log-likelihood ratio.
%In section~\ref{subsec:boundValidity}, we show that approximating the distribution of $L(x)$ using just its first two moments sufficient to calculate accurate bounds on the attack power.  We illustrate this by comparing our theoretical bound with the empirically observed maximum power for any error.

Given this approximation, and the computed mean and variance, the relationship between power $\beta$, and error $\alpha$ is
\begin{equation}\label{eq:gaussian-metrics}
	\mu_0 - z_{\alpha}\sigma_0 = \mu_1 - z_{\beta}\sigma_1
\end{equation}
where $z_s$ is the quantile at level $1 - s, 0 < s < 1$ of the standard normal distribution. This equation can be derived by equating quantiles at level $\beta$, $\alpha$ in the pool and population distribution respectively.

Substituting $\mu_0, \sigma_0, \mu_1, \sigma_1$ into~\eqref{eq:gaussian-metrics}, we derive the main result. 
\end{proof}

In case of high-dimension models, the log-likelihood ratio distribution is very close to normal distribution (as it is a sum of large number of independent random variables) and assuming it follows a normal distribution is a good approximation for most practical purposes. As we will show in the latter sections (Section~\ref{subsec:boundValidity} and Section~\ref{subsec:bound-insight}), the contribution of higher order moments to the estimates of privacy risk and understanding of the sources of information leakage is marginal. We illustrate this by comparing our theoretical bound with the empirically observed maximum power for any error and explaining the power of attacks using parameter estimation errors. 

The intuition behind our result is that the centers ($\mu_0 = \frac{C(G)}{2n}$ and $\mu_1 = -\frac{C(G)}{2n}$) of $L(x)$ under the null and alternative hypotheses are separated by a distance of $\frac{C(G)}{n}$. The overlap between the distributions is determined by variance $\frac{C(G)}{n}$ of the statistic, and the amount of the overlap between the two distributions determines the power $\beta = \Pr_{pool}\{L(x) \leq F^{-1}(\alpha)\}$ for any error $\alpha$. 

Our result generalizes that of Sankararaman~et~al.~\cite{sankararaman2009genomic} on releasing independent marginals. In their case, the released graph has no edges and nodes are binary variables. The complexity of such a graph is equal to the number of nodes $m$. Hence, for independent marginals we recover Sankararaman et al's relation:
\begin{equation}
z_{\alpha} + z_{1-\beta} = \sqrt{\frac{m}{n}}.
\end{equation}

\subsection{Insights from the bound:} \label{subsec:bound-insight}

The bound in Theorem~\ref{thm:main-result} is independent of the exact values of the data in the pool and depends only on the metadata of the model: pool size $n$, number of attributes $m$ and  model structure $G$. This implies that the analysis is robust to varying the details of the dataset, but it is expressive enough to capture and resolve questions like the following:
\begin{itemize}
	\item Which one of many model structures has the largest/smallest leakage? 
	\item What is the additional leakage caused by releasing one more attribute for each data point in the pool? 
	\item How do the dependencies among a certain group of attributes affect the leakage?
	\item How exactly does the pool size affect leakage? 
\end{itemize}

Using the bound, we can observe and quantify the effect of releasing a model in terms of its complexity $C(G)$. Releasing more parameters helps the attacker, and we also see that e.g. quadrupling $C(G)$ would double the sum $z_{\alpha} + z_{1-\beta}$, thus reducing the error or increasing the power or both. The amount of improvement depends on how large the sum already is and there are diminishing returns. In contrast, increasing the pool size $n$ has the opposite effect to increasing $C(G)$: the attack performance becomes worse. This makes sense, as a larger pool is more similar to (has more overlap with) the population, so it is more difficult for the attacker to distinguish between them.

It is also possible to see whether a heuristic attack can be improved by comparing its error and power to the ones implied by the main theorem for a given complexity and pool size. From the heuristic attack's error and power, we can compute the corresponding Standard Normal quantiles and compare their sum to $\sqrt{\frac{C(G)}{n}}$. If the sum is far from the bound, then the attack can be improved. From a defender's point of view, we can quantify the maximum leakage associated with releasing various models \textbf{without} having to train each model and \textbf{without} having to perform any attack. We can reason about the ultimate/maximum power of the attacker, e.g. one with perfect knowledge about the population, so as to guide our choice of a model to release.

It is also interesting to note that the natural complexity behind the structure of a graphical model captures its privacy risk. The difference between an estimated parameter value (calculated from the pool) and the actual parameter value (calculated from the general population) is an estimation error that leaks information about the pool. Since these estimation errors are independent across parameters of a Bayesian network, each parameter makes a separate contribution to the power of the attacker. Hence, the complexity measure defined as the number of parameters captures the potential privacy risk of the model. See Appendix~\ref{sec:fisher} for a detailed discussion on why the estimation errors of parameters in graphical models are independent.

\section{Experiments} \label{sec:experiments}
We use two methods for performing and evaluating the attack:
\begin{enumerate}
	\item \textbf{Theoretical:} Given a false positive rate and released model structure $G$, we use our main result (Theorem~\ref{thm:main-result}) to calculate the power, error, and AUC.
	
	\item \textbf{Empirical:} In empirical analysis, we vary the threshold of LR test from $-\infty$ to $+\infty$ and calculate the power at each value of false positive rate. This is the maximum possible power that can be achieved. Hence we use the power and AUC values calculated here to compare with the bound presented in Theorem~\ref{thm:main-result}.
\end{enumerate}

\subsection{Data Sets}
A summary of all the data sets which are used in our experiments is provided in Table~\ref{table:dataDescription}.

\textbf{Location:} This is a binary data set containing the Foursquare location check-ins by individuals in Bangkok~\cite{shokri2017membership}. Each record corresponds to an individual and consists of binary attributes reflecting visits to different locations.
	
\textbf{Purchase:} This is a binary data set containing information about individuals and their purchases~\cite{shokri2017membership}. Each record corresponds to an individual and each attribute represents a product. A value of 1 at attribute $j$ means that the individual purchased the product corresponding to attribute $j$.  
	
\textbf{Genome:} OpenSNP\footnote{\url{https://opensnp.org/snps}} is an open source data sharing website, where people can share their genomic data test results. We obtained the data provided by OpenSNP and considered only the individuals sequenced by 23andme. We randomly selected 1000 SNPs on chromosome 1. Individuals with more than 2 missing values were filtered out. After this pre-processing, we were left with 2497 individuals and 1000 SNPs for each individual.
	
Bayesian Networks have been used to model genome sequences in \cite{agrahari2018applications,su2013using}. We use a similar approach to model the SNPs as a Bayesian Network. Since humans are diploid, at each position, we have two bases i.e. three possible values. While releasing graphical models constructed from genomic data, we only estimate the Minor and Major Allele Frequencies. To calculate the probability of any combination, we assume independence and compute it as the product of Allele Frequencies. 

\textbf{Data augmentation and Evaluation method:} As the size of the original dataset is small for Location and Genome data, we augment the original dataset with synthetic data sampled independently from a Bayesian network with learned $\eta = 3$ (maximum number of parents per node) on the original dataset. We use the full augmented set as the general population. See Appendix~\ref{sec:pgm-appendix} for details on data synthesis, structure learning, and parameter learning in Bayesian networks. In all the experiments, the pool and reference population are sampled independently from the general population. To evaluate the attack, all available samples from the general population are used to compute power and false positives. The pool size and reference population size for experiments with the Location and Purchase datasets are 3000 and 15000 respectively. The pool size and reference population size for experiments with the Genome dataset are 1000 and 5000 respectively. We perform each experiment with 50 different and independent splits of pool and reference population and report the average statistics. \textbf{This random splitting and averaging ensures that the results are not biased by data augmentation or by a single instance of sampling the pool.}
	\begin{table}[]
		\centering
		\resizebox{0.7\columnwidth}{!}{%
			\begin{tabular}{cccc}
				\thickhline
				\textbf{Data Set} & \textbf{\# Attributes} & \textbf{Original Size} & \textbf{Augmented Dataset Size} \\ \thickhline
				Location          & 446                        & 5010               & 30000                     \\ 
				Purchase          & 600                        & 30000               & 30000                    \\ 
				Genome       & 1000                       & 2497               & 10000                     \\ \hline
			\end{tabular}
		}
		\caption{\small{\textbf{Summary of Datasets used:} As the size of the original dataset is small for Location and Genome data, we augment the original dataset with synthetic data generated independently from a Bayesian network with $\eta = 3$ learned on the original dataset, where $\eta$ is the maximum number of parents a node can have.  We use the full augmented set as the general population. See Appendix~\ref{sec:pgm-appendix} for complete details on data synthesis.}}
		\label{table:dataDescription}
		
	\end{table}

	% \begin{table}[]
	% 	\centering
	% 	\resizebox{0.7\columnwidth}{!}{%
	% 		\begin{tabular}{|c|c|c|c|}
	% 			\hline
	% 			\textbf{Data Set} & \textbf{\# Attributes} & \textbf{Original Size} & \textbf{Augmented Dataset Size} \\ \hline
	% 			Location          & 446                        & 5010               & 30000                     \\ \hline
	% 			Purchase          & 600                        & 30000               & 30000                    \\ \hline
	% 			Genome       & 1000                       & 2497               & 10000                     \\ \hline
	% 		\end{tabular}
	% 	}
	% 	\caption{\small{\textbf{Summary of Datasets used:} For Location and Genome data, we augment the original dataset with synthetic data generated independently from a Bayesian network with $\eta = 3$ learned on the original dataset, where $\eta$ is the maximum number of parents a node can have.  We use the full augmented set as the general population. See supplementary section I for complete details on data synthesis.}}
	% 	\label{table:dataDescription}
		
	% \end{table}

\begin{table*}[]
	
	\centering 
		
		\resizebox{1\columnwidth}{!}{%
			\begin{tabular}{ccccccc}
				\thickhline
				\textbf{Data set} & \textbf{No. of Nodes} & \textbf{$\eta$} & \textbf{No. of Edges} & \textbf{Complexity} & \textbf{AUC (Empirical)}  & \textbf{AUC (Theoretical)} \\ \hline \hline
				\multirow{4}{*}{Location} &\multirow{4}{*}{446}
				& 0                   & 0                    & 446     & 0.5928	& 0.6074           \\ 
				&	& 1                   & 343                  & 789   & 0.6337	& 0.6415              \\ 
				&	& 2                   & 566                  & 1222    & 0.6655 & 0.6741            \\ 
				&	& 3                   & 757                  & 1905  & 0.6998 & 0.7134              \\ \hline
				\multirow{4}{*}{Purchase} &\multirow{4}{*}{600}
				& 0                   & 0                    & 600   & 0.5700		& 0.6241              \\ 
				&	& 1                   & 496                  & 1096 & 0.6266		& 0.6654               \\ 
				&	& 2                   & 941                  & 1942  & 0.6885 & 0.7153               \\ 
				&	& 3                   & 1358                 & 3431 & 0.7541		& 0.7752               \\ \hline
				\multirow{4}{*}{Genome} &\multirow{4}{*}{1000}
				& 0                  & 0                    & 1000  & 0.6729		& 0.7602               \\ 
				&	& 1                  & 729                  & 1729 & 0.7875		& 0.8237               \\ 
				&	& 2                  & 1244                 & 2706  & 0.8495		& 0.8776            \\ 
				&	 & 3                  & 1712                 & 4323  & 0.9058		& 0.9292              \\ \thickhline
				
			\end{tabular}
		}
		\caption{\small{\textbf{AUC comparison for model structures we learned on different datasets, with different complexities}. We compare the AUC values for empirical attack with the corresponding values computed using the theoretical bound. The variable $\eta$ represents maximum number of parents a node can have in the graph. We can observe that the empirical values of AUC are closer to the bound and increase with increasing complexity of the model.}}
		\label{table:structData}
		\end{table*}

\subsection{Validity of the theoretical bound} \label{subsec:boundValidity}
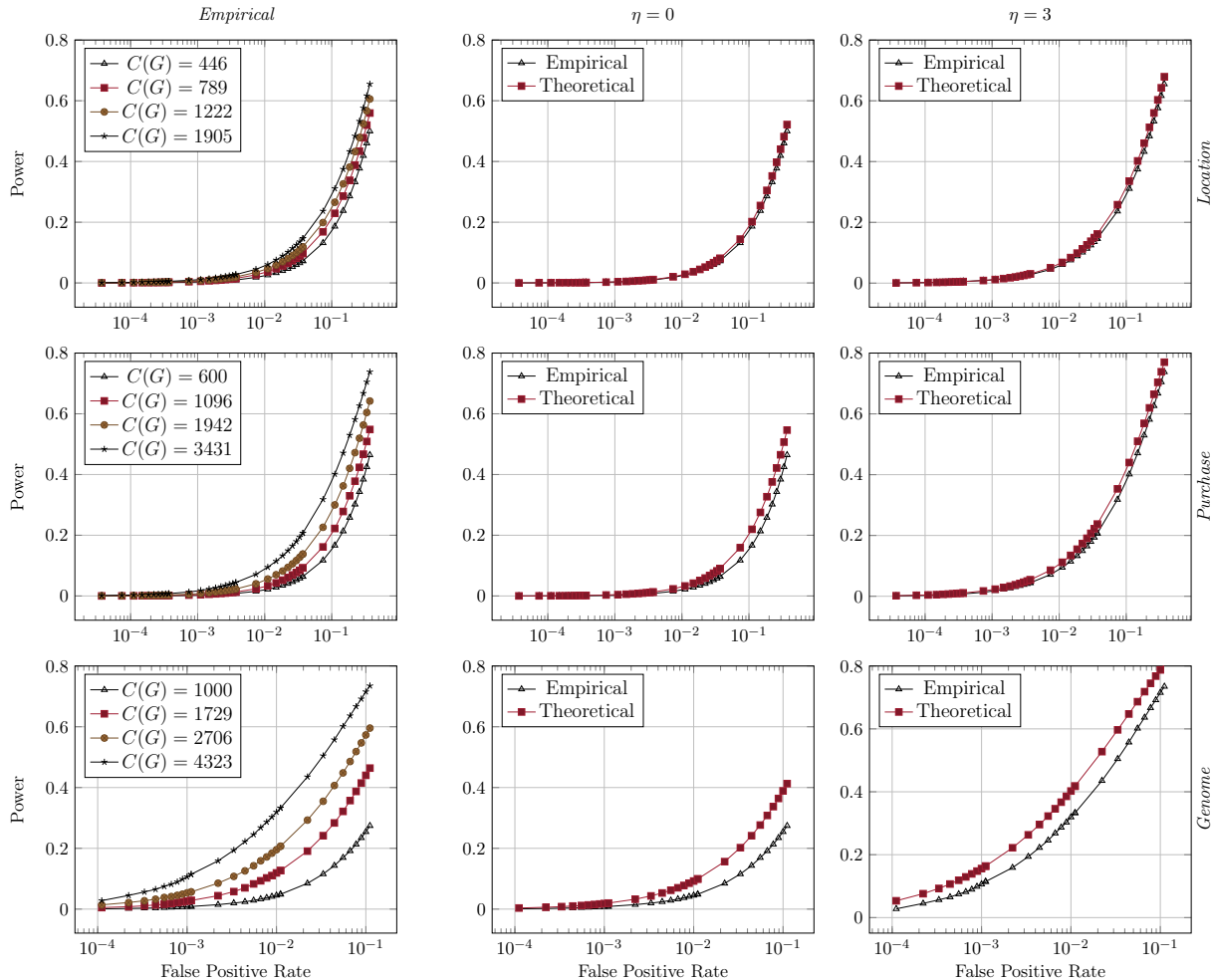
\begin{figure*}[!htb]
	\centering
	\resizebox{1\columnwidth}{!}{%
		\begin{tabular}{l@{\hskip 0.5in}lll}
			
			\begin{tikzpicture}
			\begin{semilogxaxis}
			[legend style={font=\large}, title={\em Empirical},xlabel={},ylabel={Power}, grid = major, legend entries = {$C(G) = 446$, $C(G) =789$, $C(G)= 1222$,$C(G)= 1905$}, ymax = 0.8, legend pos =  north west]
			\addplot[mark = triangle] table {"Data/Bangkok/Mean_N_Parents0.dat"};
			\addplot table {"Data/Bangkok/Mean_N_Parents1.dat"};
			\addplot table {"Data/Bangkok/Mean_N_Parents2.dat"};
			\addplot table {"Data/Bangkok/Mean_N_Parents3.dat"};
			\end{semilogxaxis}
			\end{tikzpicture} &
			
			\begin{tikzpicture}
			\begin{semilogxaxis}
			[legend style={font=\large}, title= {$\eta = 0$},xlabel={},ylabel={}, grid = major, legend entries = {Empirical, Theoretical}, ymax = 0.8,legend pos =  north west]
			\addplot[mark = triangle] table {"Data/Bangkok/Mean_N_Parents0.dat"};
			\addplot table {"Data/Bangkok/exact_theory_threshold0.dat"};
			\end{semilogxaxis}
			\end{tikzpicture} &
			
			\begin{tikzpicture}
			\begin{semilogxaxis}
			[legend style={font=\large}, title= {$\eta = 3$},xlabel={},ylabel={}, ylabel right ={\em Location}, grid = major, legend entries = {Empirical, Theoretical},ymax = 0.8, legend pos =  north west]
			\addplot[mark = triangle] table {"Data/Bangkok/Mean_N_Parents3.dat"};
			\addplot table {"Data/Bangkok/exact_theory_threshold3.dat"};
			% \addplot table[x index=0,y index=3] {"Data/Bangkok/theory-empirical-difference-1-3.dat"};
			% \addplot table[x index=0,y index=6] {"Data/Bangkok/theory-empirical-difference-1-3.dat"};
			\end{semilogxaxis}
			\end{tikzpicture}
			
			\\
			\begin{tikzpicture}
			\begin{semilogxaxis}
			[legend style={font=\large},title= {},xlabel={},ylabel={Power}, grid = major, legend entries = {$C(G)= 600$, $C(G)= 1096$, $C(G)= 1942$, $C(G)= 3431$}, ymax = 0.8, legend pos =  north west]
			\addplot[mark = triangle] table {"Data/Purchase/Mean_N_Parents0.dat"};
			\addplot table {"Data/Purchase/Mean_N_Parents1.dat"};
			\addplot table {"Data/Purchase/Mean_N_Parents2.dat"};
			\addplot table {"Data/Purchase/Mean_N_Parents3.dat"};
			\end{semilogxaxis}
			\end{tikzpicture} &
			
			\begin{tikzpicture}
			\begin{semilogxaxis}
			[legend style={font=\large}, title= {},xlabel={},ylabel={}, grid = major, legend entries = {Empirical, Theoretical}, ymax = 0.8,legend pos =  north west]
			\addplot[mark = triangle] table {"Data/Purchase/Mean_N_Parents0.dat"};
			\addplot table {"Data/Purchase/exact_theory_threshold0.dat"};
			\end{semilogxaxis}
			\end{tikzpicture} &
			
			\begin{tikzpicture}
			\begin{semilogxaxis}
			[legend style={font=\large}, title= {},xlabel={},ylabel={}, grid = major, ylabel right ={\em Purchase}, legend entries = {Empirical, Theoretical},ymax = 0.8, legend pos =  north west]
			\addplot[mark = triangle] table {"Data/Purchase/Mean_N_Parents3.dat"};
			\addplot table {"Data/Purchase/exact_theory_threshold3.dat"};
			% \addplot table[x index=0,y index=3] {"Data/Purchase/theory-empirical-difference-1-3.dat"};
			% \addplot table[x index=0,y index=6] {"Data/Purchase/theory-empirical-difference-1-3.dat"};
			\end{semilogxaxis}
			\end{tikzpicture}
			\\
			
			\begin{tikzpicture}
			\begin{semilogxaxis}
			[legend style={font=\large},title= {},xlabel={False Positive Rate},ylabel={Power}, grid = major, legend entries = {$C(G)= 1000$, $C(G)= 1729$, $C(G)= 2706$, $C(G)= 4323$ }, ymax = 0.8, legend pos =  north west]
			\addplot[mark = triangle] table {"Data/Genome_1000/Mean_N_Parents0.dat"};
			\addplot table {"Data/Genome_1000/Mean_N_Parents1.dat"};
			\addplot table {"Data/Genome_1000/Mean_N_Parents2.dat"};
			\addplot table {"Data/Genome_1000/Mean_N_Parents3.dat"};
			\end{semilogxaxis}
			\end{tikzpicture} &
			
			\begin{tikzpicture}
			\begin{semilogxaxis}
			[legend style={font=\large}, title= {},xlabel={False Positive Rate},ylabel={}, grid = major, legend entries = {Empirical, Theoretical }, ymax = 0.8,legend pos =  north west]
			\addplot[mark = triangle] table {"Data/Genome_1000/Mean_N_Parents0.dat"};
			\addplot table {"Data/Genome_1000/exact_theory_threshold0.dat"};
			\end{semilogxaxis}
			\end{tikzpicture} &
			
			\begin{tikzpicture}
			\begin{semilogxaxis}
			[legend style={font=\large}, title= {},xlabel={False Positive Rate},ylabel={}, ylabel right ={\em Genome}, grid = major, legend entries = {Empirical, Theoretical},ymax = 0.8, legend pos =  north west]
			\addplot[mark = triangle] table {"Data/Genome_1000/Mean_N_Parents3.dat"};
			\addplot table {"Data/Genome_1000/exact_theory_threshold3.dat"};
			% \addplot table[x index=0,y index=3] {"Data/Genome_1000/theory-empirical-difference-1-3.dat"};
			% \addplot table[x index=0,y index=6] {"Data/Genome_1000/theory-empirical-difference-1-3.dat"};
			\end{semilogxaxis}
			\end{tikzpicture}

		\end{tabular}
	}
	\caption{\small{\textbf{Power of Attack on Real world data:} The effect of model complexity on the power of the attack is shown in the first column. When graphical models of increasing complexity are released, the power of attack increases, as we have more parameters that can leak information.  In the second and third columns, we compare the observed powers with their corresponding theoretical bounds for models with $\eta = 0$ and $\eta = 3$ respectively. We can clearly observe that the two curves (empirically observed power and theoretical bound) are very close to each other demonstrating the validity and usefulness of the bound.}}
	\label{table: resultsPlot}
\end{figure*}

We first present how the complexity of released graphical model affects the power of tracing attack. Table~\ref{table:structData} and Figure~\ref{table: resultsPlot} (column 1) show the AUC and power respectively of the tracing attack when models of various complexities are released. In Figure~\ref{table: resultsPlot} (columns 2, 3), we compare the observed power of tracing attack with the bound from Theorem~\ref{thm:main-result}. Columns 2 and 3 correspond to releasing Bayesian Networks learned with $\eta = 0$ and $\eta = 3$ respectively. The variable $\eta$ represents maximum number of parents a node can have in the graph. 

As shown in Table~\ref{table:structData}, the AUC values are comparatively smaller for the Purchase data set compared to that of the Genomes dataset. This is because, in case of Purchase data, we only have 600 attributes for a pool size of 3000. In case of Genomic data, we have 1000 attributes for a pool size of 1000. Even then, on Purchase data, if a Bayesian network with $\eta = 3$ is released, we can achieve an AUC value of approximately 0.75, compared to 0.57 when only marginals are released. The complexity of a graphical model represents the number of independent parameters in the model. Each of these parameters is learned from individuals in the pool and hence can leak more information about membership in the pool. This leakage contributes to the power of the tracing attack. Hence, we can confirm that the higher the complexity of the released model, the higher the power of the tracing attack.

We can clearly observe that the empirical and theoretical power in column 2 and 3 of Figure~\ref{table: resultsPlot} are very close to each other except for the case of $\eta = 0$ on genome data. When $\eta = 0$, the model cannot capture any dependency in the data. If the released model does not capture all the dependencies among attributes in data (under fitted), then estimation errors of parameters in released graphical model become correlated. This effectively reduces the amount of information available to perform membership inference. Overall, we can see that the theoretical bound can capture the empirically observed power very effectively. \textbf{The observed power of tracing attack is close to the value calculated from the theoretical bound demonstrating its validity and usefulness.}

\section{Related Work}
Homer et al. \cite{homer2008resolving} developed a statistical test based on likelihood ratio for inferring the presence of a genome sequence, given the Allele frequencies. Sankararaman et al. \cite{sankararaman2009genomic} extend this work and provide tight bounds on the power of tracing attack for any adversary. The test statistic in \cite{sankararaman2009genomic} is based on likelihood ratio. This was further extended for continuous Gaussian variables (Micro RNA) data in \cite{backes2016membership}. Similar attacks on genomic data using statistics published in association studies are performed in \cite{shringarpure2015privacy,wang2009learning}. Dwork et al. \cite{dwork2015robust} take a different approach and provide a generic framework for tracing attacks based on distance metric, when noisy statistics are released. In \cite{im2012sharing}, the authors use a correlation statistic to perform a tracing attack against regression coefficients from quantitative phenotypes. All of \cite{homer2008resolving,sankararaman2009genomic,dwork2015robust,backes2016membership,im2012sharing} assume independence among data attributes, whereas our work addresses the case of dependent attributes.

Shokri et al. \cite{shokri2017membership} perform membership inference attacks against black-box machine learning models. The adversary in \cite{shokri2017membership} constructs \textit{shadow models} that mimic the behavior of the target model. The attack is treated as a binary classification problem and the decision rule is a machine learning model trained on data from the shadow models. Salem et al.  \cite{salem2018ml} follow a similar framework as \cite{shokri2017membership} but relax certain assumptions on the knowledge and power of adversary. Similar attacks were performed against aggregate location data \cite{pyrgelis2017knock}, generative adversarial networks \cite{hayes2018logan} and in a collaborative learning setting \cite{melis2018exploiting,nasr2019comprehensive}. These works provide empirical analysis of tracing attack on complex models. A theoretical formulation of Bayes-optimal attack for membership inference against neural networks was given in \cite{sablayrolles2019white}, which shows that existing techniques based on shadow models \cite{shokri2017membership} are approximations of this optimal attack. In \cite{sablayrolles2019white} it is shown that the power of optimal attack on black box models is the same as that on white box models, but no bound on this power is provided. We present a theoretical bound on the power of the attack, which is independent of the training data and the auxiliary knowledge of the adversary.

Differential Privacy \cite{dwork2006calibrating} has been accepted as the de facto standard notion of privacy. Zhang et al. \cite{zhang2017privbayes} learn a Bayesian network in a differentially private way and then use a noisy version of it to generate synthetic data. The authors in \cite{bindschaedler2017plausible} introduce the notion of plausible deniability and propose a mechanism that achieves it by generating synthetic data that is statistically similar to the given input data set. By definition, differential privacy decreases the power of tracing attack and its effect on our bound is discussed in Appendix~\ref{sec:discussion-dp}.

% \textbf{Other Privacy Attacks on PGMs:} A polynomial time algorithm to reconstruct databases from noisy statistics was provided in \cite{dinur2003revealing}. Hidden attributes of genomic data were reconstructed in \cite{ayday2017inference} by modeling it as a Markov chain. A detailed survey of attacks on private data is provided in \cite{dwork2017exposed}.

% \textbf{Applications of PGM:} Bayesian networks were used to classify hematological malignancies along with co-expression networks in  \cite{agrahari2018applications}. The Bayesian network was learned on gene expression data. Certain dependencies in this learned Bayesian network were made public through the paper. The applicability of Bayesian networks to model the effect of environment and genes on diseases was discussed in \cite{su2013using}. The bladder cancer study data in \cite{su2013using} was modeled as a Bayesian network. Bayesian Networks are also being used as intelligent tutoring systems by modeling students \cite{kaser2017dynamic}. The increasing use of graphical models in sensitive domains raises the significance of our privacy analysis on releasing high dimensional models learned on private data. 
% !TEX root = ../main.tex

\section{Summary}
We provide a theoretical analysis of tracing attacks against probabilistic graphical models to address the existing gap between theoretical analysis for simple average statistics on data with independent attributes and empirical demonstrations for complex models on data with correlated attributes. Our bound quantifies the maximum attack performance measured with the error (false positive rate) and power (true positive rate) of a likelihood-ratio test. We experimentally validate and complement our results using sensitive datasets - location check-ins, purchase history, genomic data.
% !TEX root = ../main.tex
\section*{Acknowledgments}
 This work is partially supported by the Singapore Ministry of Education Academic Research Fund, R-252-000-660-133, the NUS Early Career Research Award (NUS ECRA), grant number NUS ECRA FY19 P16, and the National Research Foundation, Singapore under its Strategic Capability Research Centres Funding Initiative. 
% \bibliographystyle{abbrv}
% \bibliography{references}

\clearpage
\newpage
\appendix
% !TEX root = ../main.tex
\section{Table of Notations}
\begin{table}[!h]
	\centering
	\resizebox{0.7\columnwidth}{!}{%
		\begin{tabular}{ccc}
			\hline
			\textbf{Symbol} & \textbf{Description}\\
			$m$          		& Number of attributes\\
			$n$   		    	& Pool (training set) size\\ 
			$\pgmpool$			& Released Model\\
			$\pgmpop$ 		& Population Model\\ 
			$X_i$				& Random variable for attribute $i$\\
			$x_i$				& A particular value for $X_i$\\
			$V(X_i)$	 		& Set of possible values of attribute $i$\\
			$Pa_{X_i}^G$		& Set of random variables that are parents of node $X_i$ in $G$\\
			$p_i^v$ 	& $\Pr(x_i = 1 | Pa_{X_i}^G = v; \theta)$\\
			$\hat{p}_i^v$ 	& $\Pr(x_i = 1 | Pa_{X_i}^G = v; \hat{\theta})$\\
			$C(G)$		& Complexity of $G$ (number of independent parameters)\\
			$\eta$		& Maximum number of parents per node in $G$\\
			$L(x)$		& Log-likelihood ratio statistic \eqref{eq:log-likelihood-ratio} for data sample $x$ \\
			$F$			& CDF of $L(x)$ over the general population (under $\hout$)\\
			$\alpha$    & Error (False Positive Rate) of LR tracing attack \\
			$\beta$     & Power (True Positive Rate) of LR tracing attack \\ 
			$z_s$       & Quantile at level $1 - s$ of the Standard Normal distribution \\ \hline
		\end{tabular}
	}
	\caption{\textbf{Notations}}
	 \vspace{-2em}

	\label{table:notations}
\end{table}

\clearpage
\newpage
% !TEX root = ../main.tex
\section{Derivation of mean and variance of the likelihood ratio}
\label{subsec:derivation-mean-variance}
We compute the mean and variance of $L(x)$ under the two hypotheses. We sketch the proof for the mean $\E(L)$ under the population hypothesis, followed by the variance $\Var(L)$. Similar calculations apply for the pool hypothesis.

Let the target $x$ have the feature vector $(x_1, x_2, \ldots, x_m)$, and let us assume, for now, that all attributes are binary: $x_i \in \{0, 1\}, i = 1, \ldots, m$. In Appendix~\ref{app:categorical} we generalize to attributes that can take more than two values. We can take advantage of the Bayesian network decomposition to write the log-likelihood ratio for $x$ as follows:
\begin{align*}
L(x) =& \log\left[\frac{\Pr(x ; \pgmpop)} {\Pr(x ; \pgmpool)}\right]
= \sum_{i=1}^{m} L_i \numberthis
\end{align*}
where $L_i$ is the contribution of $X_i$ to the likelihood ratio, as defined as:
\begin{align*}
L_i =& \log\left(\frac{\Pr[X_i | \parpop]}{\Pr[X_i | \parpool]}\right) \\
=& \sum_{v \in V(Pa_{X_i}^G)} \ind{i}{v} \underbrace{\left(x_i \log \frac{p_i^v}{\hat{p}_i^v} + (1 - x_i) \log \frac{1 - p_i^v}{1 - \hat{p}_i^v} \right)}_\text{$L_i^v$}\\
=& \sum_{v \in V(Pa_{X_i}^G)} \ind{i}{v} L_i^v \numberthis \label{eq:likelihood-contribution}
\end{align*}
where $p_i^v = \Pr\{X_i = 1 |Pa_{X_i}^G = v;\theta\}$, and similarly $\hat{p}_i^v = \Pr\{X_i = 1 |Pa_{X_i}^{G} = v;\hat{\theta}\}$. The notation $\ind{i}{v}$ is an indicator variable for a particular assignment of values to the parent nodes of $X_i$, i.e. $\ind{i}{v} = 1$ if $Pa_{X_i}^G = v$ and 0 otherwise. The sum ranges over $|V(Pa_{X_i}^G)|$ terms $L_i^v$, one for each element of $V(Pa_{X_i}^G)$.

The parameters $\theta$ and $\hat{\theta}$ are estimated from data (reference population and pool, respectively). By the central limit theorem, the distribution of such an estimate converges to a Gaussian around the mean value of the estimate as the number of data samples increases.
In our derivations of the mean and variance, we use this approximation in \eqref{eq:approx-nis} and \eqref{eq:approx-hat-pis}. By the Berry-Esseen theorem \cite{berry1941accuracy,esseen1942liapunoff}, the rate of convergence to the Gaussian is $O(\frac{1}{\sqrt{n}})$ if the third moment of the random variable being sampled is finite. In our case this condition is true, because each random variable can only take a finite number of possible finite values.

We compute the mean and variance of $L(x)$ as follows:
\begin{subequations}
\begin{align}
&\E_{pop}(L)    =  \frac{C(G)}{2n} + O(C(G) n^{-2})\\
&\E_{pool}(L)   = -\frac{C(G)}{2n} + O(C(G) n^{-2})\\
&\Var_{pop}(L)  =  \frac{C(G)}{n}  + O((C(G))^2 n^{-2})\\
&\Var_{pool}(L) =  \frac{C(G)}{n}  + O((C(G))^2 n^{-2}).
\end{align}\label{eq:mean-variance-approximate}
\end{subequations}
\begin{proof}[Proof sketch - Mean under $\hout$]
	The mean $\E_{pop}(L)$ can be computed as follows:
	\begin{align*}
	\E_{pop}(L) =& \sum_{i=1}^{m} \E_{pop}(L_i)\\
	=& \sum_{i=1}^{m} \sum_{v} \E_{pop}(\ind{i}{v} L_i^v) \numberthis
	\end{align*} 
	Using approximation \eqref{eq:approx-E-Lis} in Appendix Section~\ref{app:mean}, we compute $\E_{pop}(\ind{i}{v} L_i^v) \approx \frac{1}{2n} + O(n^{-2}) $.
	Since the total number of $L_i^v$ parameters is $C(G) = \sum_{i=1}^{m} |V(Pa_{X_i}^G)|$, we conclude that 
	\begin{equation}\label{eq:mean-inproof}
	\E_{pop}(L) = \frac{C(G)}{2n} + O(C(G)n^{-2}).
	\end{equation}
	
\end{proof}
\begin{proof}[Proof sketch - Variance under $\hout$]
	By definition,
	\begin{equation}
	\Var_{pop}(L) = \E_{pop}[L^2] - (\E_{pop}[L])^2.
	\end{equation}
	The latter term $(\E_{pop}[L])^2$ is the square of the mean, which we compute in \eqref{eq:mean-inproof}.
	The former term $\E_{pop}[L^2]$ decomposes as follows:
	\begin{equation}\label{eqn:variance-eqn1}
	\E_{pop}[L^2] = \sum_{i=1}^{m} \E_{pop}[L_i^2] + 2\sum_{1\leq i < j \leq m} \E_{pop}[L_i L_j]
	\end{equation}
	
	We compute $\E_{pop}[L_i^2]$ by expanding $\E_{pop}[(\sum_{v} \ind{i}{v} L_i^v)^2]$. Then, approximation \eqref{eq:approxTSquare} in Appendix~\ref{app:variance} gives us that each square term $\E_{pop}[(\ind{i}{v} L_i^v)^2]$ is approximately equal to $\frac{1}{n}$. As for the product terms in the expansion, each term multiplies two different indicator variables $\ind{i}{v}$ and $\ind{i}{v'}$ with $v \neq v'$. Because at most one of the two is equal to 1, all product terms will be zero. Hence $\E_{pop}[L_i^2] = |V(Pa_{X_i}^G)| \times \frac{1}{n}$.
	
	The number of joint terms $\E_{pop}[L_i L_j]$ is $O((C(G))^2)$. From the approximation in Appendix Section~\ref{app:Ti2-TiTj} for $\E_{pop}[L_i L_j]$, each of these terms is equal to $\frac{1}{4n^2}$ with error term $O(n^{-2})$. Hence, the value of $\E_{pop}[L^2]$ is 
	\begin{equation}
	E_{pop}[L^2] = \frac{C(G)}{n} + \frac{(C(G))^2}{4n^2} + O((C(G))^2 n^{-2}).
	\end{equation}
	We conclude that the variance is
	\begin{align*}
	\Var_{pop}(L) =& \E_{pop}[L^2] - (\E_{pop}[L])^2\\
	=& \frac{C(G)}{n} + \frac{(C(G))^2}{4n^2} + O((C(G))^2 n^{-2}) - \left(\frac{C(G)}{2n} + O(C(G)n^{-2})\right)^2\\
	=& \frac{C(G)}{n} + O((C(G))^2 n^{-2}) \numberthis
	\end{align*}
\end{proof}

Although we haven't provided the calculation here, it is possible to calculate the exact value of the $O((C(G))^2 n^{-2})$ term from the released model. As a simple example, in appendix \ref{app:naive-bayes}, we calculate the exact value of this $O((C(G))^2n^{-2})$ term, when the released model is a Naive Bayes model.

\subsection{Distribution of the log-likelihood ratio:}

To compute the distribution of $L(x)$, the log-likelihood ratio of the parameter vector estimate with the actual value of the parameter vector in a graphical model, we need to understand what parameters contribute to the likelihood ratio given a data sample. As shown in equation~\eqref{eq:likelihood-contribution}, the parameter that contributes to the likelihood ratio for an attribute is determined by the value taken by its parent node. The contribution of attribute $X_i$ to the likelihood function, denoted by $L_i$, is computed as:
\begin{align*}
L_i = \sum_{v \in V(Pa_{X_i}^G)} \ind{i}{v} L_i^v 
\end{align*}

Hence the distribution of $L_i$ is a mixture of the distributions of $L_i^v$, where the mixing probabilities are determined by the distribution of the parent nodes (hence the dependence of $L(x)$'s distribution on the probability distribution that generated the data). The distribution of $L_i^v$, the log-likelihood ratio for the estimate of a single parameter value, is asymptotically a chi-squared distribution with degree of freedom 1 (from Wilks' theorem). Hence, the exact distribution of log-likelihood ratio is a sum of mixture of chi-squared distributions, where the mixing distribution is dependent on the distribution that generated the data. In case of high-dimension models, this log-likelihood ratio distribution is very close to normal distribution (as it is a sum of large number of independent random variables (which are sum of mixtures themselves)). Hence, using the first two moments, that do not depend on the exact distribution of the sensitive data, is sufficient to produce a generic data-independent upper-bound on the privacy risk of learning the graphical model.

% !TEX root = ../main.tex
\section{Number of Samples for Estimating Conditional Probabilities}

We use $\hat{p}_i^v$ to denote the estimated conditional probability that $X_i = 1$, given that the values of the activator variables are $Pa_{X_i}^G=v$.
The number of samples $n_i^v$ used to compute $\hat{p}_i^v$ are approximately Gaussian around $np_{v}$ ($n$ is the pool size, and $p_{v}$ is the probability of $Pa_{X_i}^G = v$ in the general population):
\begin{equation}\label{eq:approx-nis}
n_i^v \approx np_{v} + \sqrt{np_{v} (1 - p_{v})} Z_1,
\end{equation} 
where $Z_1$ is a standard Gaussian random variable. In parallel, the value of $\hat{p}_i^v$ is also approximately Gaussian around the true value $p_i^v$:
\begin{equation}\label{eq:approx-hat-pis}
\hat{p}_i^v \approx p_i^v + \sqrt{\frac{p_i^v(1-p_i^v)}{n_i^v}}Z_2,
\end{equation}
where $Z_2$ is a standard Gaussian random variable.

Using these two approximations, we now prove the results required for derivation of LR statistic mean and variance.
\section{Approximation for mean derivation}\label{app:mean}% KL divergence

As explained in section~\ref{subsec:derivation-mean-variance}, to compute the mean of the likelihood ratio we need the average contribution of each $L_i^v$ i.e. value of $E_{pop}[\ind{i}{v} L_i^v]$. Here we prove that $E_{pop}[\ind{i}{v} L_i^v]$, when the expectation is over population is approximately equal to $\frac{1}{2n}$. When expectation is over pool, the derivation steps are  similar and the value is $-\frac{1}{2n}$.

\begin{lemma}
	We will prove the following result:
	\begin{equation}\label{eq:approx-E-Lis}
	E_{pop}[\ind{i}{v} L_i^v] \approx \frac{1} {2n} \left( 1 + \frac{1 - p_{v}}{np_{v}}\right)
	\end{equation}
\end{lemma}	
\begin{proof}
	We first observe that
	\begin{align*}
	E_{pop}[\ind{i}{v} L_i^v] =& E_{\hat{p}_i^v} \left[E_x\left[\ind{i}{v} L_i^v \mid \hat{p}_i^v\right] \right] \\
	=& p_{v} E_{\hat{p}_i^v} \left[p_i^v \log \frac{p_i^v}{\hat{p}_i^v} + (1 - p_i^v) \log \frac{1 - p_i^v}{1 - \hat{p}_i^v} \right],
	\end{align*}
	and now all we need to show is that
	\begin{equation}
	\begin{aligned}[b]
	E_{Z_1,Z_2} \left[p_i^v \log \frac{p_i^v}{\hat{p}_i^v} + (1 - p_i^v) \log \frac{1 - p_i^v}{1 - \hat{p}_i^v} \right] \approx \frac{1} {2np_{v}} \left( 1 + \frac{1 - p_{v}}{np_{v}} \right) \label{eq:approx-E_pis_logpis}
	\end{aligned}
	\end{equation}
	We approximate $\hat{p}_i^v$ with \eqref{eq:approx-hat-pis} and we use the Taylor expansion of $\log(1 + x) \approx x - \frac{1}{2}x^2$:
	\begin{align*}
	p_i^v \log \frac{p_i^v}{\hat{p}_i^v} \approx& - p_i^v \log \frac{p_i^v + \sqrt{ \frac{p_i^v(1-p_i^v)} {n_i^v}} Z_2} {p_i^v}\\
	=& - p_i^v \log \left(1 + \sqrt{\frac{1-p_i^v} {n_i^v p_i^v}} Z_2 \right)\\
	\approx& - p_i^v \left(\sqrt{\frac{1-p_i^v} {n_i^v p_i^v}} Z_2 - \frac{1-p_i^v} {2n_i^v p_i^v} Z_2^2\right)\\
	=& - \sqrt{\frac{ p_i^v(1-p_i^v)} {n_i^v}} Z_2 + \frac{1-p_i^v} {2n_i^v} Z_2^2 \numberthis \label{eq:approx_pis_log}
	\end{align*}
	Similarly,
	\begin{equation}\label{eq:approx_1mpis_log}
	(1 - p_i^v) \log \frac{1 - p_i^v}{1 - \hat{p}_i^v} \approx - \sqrt{\frac{ p_i^v(1-p_i^v)} {n_i^v}} Z_2 + \frac{p_i^v} {2n_i^v} Z_2^2
	\end{equation}
	
	Adding \eqref{eq:approx_pis_log} and \eqref{eq:approx_1mpis_log}, we have
	\begin{equation}\label{eq:pis_logpis}
	p_i^v \log \frac{p_i^v}{\hat{p}_i^v} + (1 - p_i^v) \log \frac{1 - p_i^v}{1 - \hat{p}_i^v} \approx - 2\sqrt{\frac{ p_i^v(1-p_i^v)} {n_i^v}} Z_2 + \frac{1} {2n_i^v} Z_2^2
	\end{equation}

	Taking the expectation $E_{Z_2}[.]$, and recalling that $E[Z_2] = 0$ and  $E[Z_2^2] = 1$, we have
	\begin{align*}
	E_{Z_1,Z_2} \left[p_i^v \log \frac{p_i^v}{\hat{p}_i^v} + (1 - p_i^v) \log \frac{1 - p_i^v}{1 - \hat{p}_i^v} \right] = & E_{Z_1}[E_{Z_2}[\ldots | Z_1]]\\
	\approx& E_{Z_1}\left[\frac{1} {2n_i^v}\right] \numberthis \label{eq:approx-iter}
	\end{align*}
	
	We now approximate $n_i^v$ with \eqref{eq:approx-nis} and we use the Taylor expansion of $\frac{1}{1 + x} \approx 1 - x + x^2$:
	\begin{align*}
	\frac{1} {2n_i^v} \approx& \frac{1} {2 (np_{v} + \sqrt{np_{v} (1 - p_{v})} Z_1)}\\
	=& \frac{1} {2np_{v}} \frac{1}{1 + \sqrt{\frac{1 - p_{v}}{np_{v}}} Z_1}\\
	\approx& \frac{1} {2np_{v}} \left( 1 - \sqrt{\frac{1 - p_{v}}{np_{v}}} Z_1 + \frac{1 - p_{v}}{np_{v}} Z_1^2 \right) \numberthis \label{eq:approx-1_over_nis}
	\end{align*}
	Taking the expectation $E_{Z_1}[.]$, and recalling that $E[Z_1] = 0$ and  $E[Z_1^2] = 1$, we have our final result:
	\begin{align*}
	E_{Z_1,Z_2} \left[p_i^v \log \frac{p_i^v}{\hat{p}_i^v} + (1 - p_i^v) \log \frac{1 - p_i^v}{1 - \hat{p}_i^v} \right] \approx \frac{1} {2np_{v}} \left( 1 + \frac{1 - p_{v}}{np_{v}} \right)
	\end{align*}
\end{proof}
% !TEX root = ../main.tex
\section{Approximation for variance derivation}\label{app:variance}
For calculating the variance of likelihood ratio, we need the expected values of $L_i^2$ and $L_iL_j$. Here we first prove the below approximation and use it to calculate $E(L_i^2)$ and $E(L_iL_j)$. As explained in section~\ref{subsec:derivation-mean-variance}, using these values of $E(L_i^2)$ and $E(L_iL_j)$ in equation \ref{eqn:variance-eqn1} we get the variance of LR statistic.

\begin{lemma}
	We will prove the following approximation:
	\begin{equation}\label{eq:approx-E_pis_logpis_sq}
	\begin{aligned}[b]
	& \E_{\hat{p}_i^v}\left[p_i^v \left( \log \frac{p_i^v}{\hat{p}_i^v}\right)^2 + (1 - p_i^v) \left( \log \frac{1 - p_i^v}{1 - \hat{p}_i^v} \right)^2 \right] \approx  \frac{1} {np_v}  \left( 1 + \frac{1 - p_v}{np_v} \right)
	\end{aligned}
	\end{equation}
\end{lemma}
\begin{proof}	
	Using approximation \eqref{eq:approx_pis_log}
	\begin{flalign*}
	\E_{\hat{p}_i^v}\left[p_i^v \left( \log \frac{p_i^v}{\hat{p}_i^v}\right)^2 \right] \approx& \E_{Z_1, Z_2}\left[\frac{1}{p_i^v}\left(- \sqrt{\frac{ p_i^v(1-p_i^v)} {n_i^v}} Z_2 + \frac{1-p_i^v} {2n_i^v} Z_2^2\right)^2\right]\\
	=& \frac{1}{p_i^v} \E_{Z_1, Z_2}\left[\frac{p_i^v(1-p_i^v)} {n_i^v}Z_2^2 + \left(\frac{1-p_i^v} {2n_i^v}\right)^2 Z_2^4 -2\sqrt{\frac{p_i^v(1-p_i^v)} {n_i^v}} \frac{1-p_i^v}{2n_i^v} Z_2^3  \right]\\
	=& \frac{1}{p_i^v} \E_{Z_1}\left[\frac{p_i^v(1-p_i^v)} {n_i^v} + 3\left(\frac{1-p_i^v} {2n_i^v}\right)^2  \right]&\\
	\approx& (1-p_i^v)\E_{Z_1}\left[\frac{1}{n_i^v}\right]&\\
	\approx& (1-p_i^v)\E_{Z_1}\left[\frac{1} {np_v} \left( 1 - \sqrt{\frac{1 - p_v}{np_v}} Z_1 + \frac{1 - p_v}{np_v} Z_1^2 \right)\right]&\\
	=& \frac{1-p_i^v} {np_v}  \left( 1 + \frac{1 - p_v}{np_v} \right) \numberthis \label{eq:approx_square_log}
	\end{flalign*}
	Similar to \eqref{eq:approx_square_log}, we have:
	\begin{equation}
	\E_{\hat{p}_i^v}\left[(1 - p_i^v) \left( \log \frac{1 - p_i^v}{1 - \hat{p}_i^v} \right)^2\right] \approx \frac{p_i^v} {np_v}  \left( 1 + \frac{1 - p_v}{np_v} \right)
	\end{equation}
	The desired result follows.
\end{proof}

\subsection{Approximation of $\E_{pop}[L_i^2]$}\label{app:Ti2-TiTj}
We approximate $\E_{pop}[L_i^2]$ as:
\begin{align*}
\E_{pop}[L_i^2] =& \E_{pop}\left[ \left( \sum_{v} \ind{i}{v} L_i^v \right)^2 \right]\\
=& \E_{\hat{p}_i^v} \left[ \E \left[ \left( \sum_{v} \ind{i}{v} L_i^v \right)^2 \middle| \hat{p}_i^v \right] \right]\\
=& \sum_{v} p_v \E_{\hat{p}_i^v}[(L_i^v)^2]\\
=& \sum_{v} p_v \E_{\hat{p}_i^v}\left[p_i^v \left( \log \frac{p_i^v}{\hat{p}_i^v}\right)^2 + (1 - p_i^v) \left( \log \frac{1 - p_i^v}{1 - \hat{p}_i^v} \right)^2 \right]\\
\approx& \sum_{v} \frac{1} {n} \left(1 + \frac{1 - p_v}{np_v}\right) \text{(from approximation \eqref{eq:approx-E_pis_logpis_sq})}\\
=& \frac{1}{n}|V(Pa_{X_i}^G)| + \frac{1}{n^2}\sum_{v}\frac{1 - p_v}{p_v} \numberthis \label{eq:approxTSquare}
\end{align*}	

Combining the definition of complexity with equation \eqref{eq:approxTSquare}, we have:

\begin{equation}
\sum_{i=1}^{m} \E_{pop}[L_i^2] \approx \frac{C(G)}{n} + \frac{1}{n^2}\sum_{i=1}^{m}\sum_{v}\frac{1 - p_v}{p_v}
\end{equation}

\subsection{Approximation of $\E_{pop}[L_i L_j]$}

There are three possible cases while finding the value of $\E[L_i L_j]$. The random variables $X_i$ and $X_j$ might not have any common parents, might have some common parents or one is the parent of other. We start with the case in which $X_i$ and $X_j$ have no common parents. Let $p(v_i, v_j)$ represent the joint probability of $Pa_{X_i}^G = v_i$ and $Pa_{X_j}^G = v_j$.  \\
\begin{flalign*}
\E_{pop}[L_i L_j] =& \E_{pop}\left[\left(\sum_{v} \ind{i}{v_i} L_i^v\right) \left(\sum_{v_j} \ind{j}{v_j} L_j^v\right)\right]&\\
&= \E_{pop}\left[\sum_{v_i, v_j} \ind{i}{v_i} \ind{j}{v_j} L_i^v L_j^v \right] &\\
&= \sum_{v_i, v_j} \E_{pop}\left[\ind{i}{v_i} \ind{j}{v_j} L_i^v L_j^v \right] &\\
&= \sum_{v_i, v_j} p(v_i, v_j) \E_{pop}\left[L_i^v L_j^v \right]&\\
&\approx \sum_{v_i, v_j} p(v_i, v_j) \times \frac{1}{2np_{v_i}} \times \frac{1}{2np_{v_j}} \text{(from \eqref{eq:approx-E_pis_logpis})} &\\
&= \sum_{v_i, v_j} \frac{1}{4n^2} \times \frac{p(v_i, v_j)}{p_{v_i}p_{v_j}}  \numberthis \label{eq:NoParents}
\end{flalign*}

In the case where $X_i$ and $X_j$ have common parents $S_{ij}$, let $S_i$ represent the parents exclusive to $X_i$ and $S_j$ represent parents exclusive to $X_j$. Let $p(v_i, v_j, v_{ij})$ represent the joint probability of $Pa_{X_i}^G = v_i$ and $Pa_{X_j}^G = v_j$ and common parent of $X_i$ and $X_j$, $Pa_{X_{i,j}}^G = v_{ij}$. 

\begin{align*}
\E_{pop}[L_i L_j] =& \E_{pop}\left[\left(\sum_{v_i, v_{ij}} \ind{i}{v_i} \ind{ij}{v_{ij}} L_i^v\right) \left(\sum_{v_j,v_{ij}} \ind{j}{v_j} \ind{ij}{v_{ij}} L_j^v\right)\right]\\
=& \E_{pop}\left[\sum_{v_i, v_j,v_{ij}} \left(\ind{i}{v_i} \ind{j}{v_j} \ind{ij}{v_{ij}} L_i^v L_j^v \right) \right] \\
% =& \sum_{v_i, v_j,v_{ij}} \E\left[1_{v_i} 1_{v_j} 1_{v_{ij}} L_i^v L_j^v \right] \\
=& \sum_{v_i, v_j,v_{ij}} p(v_i, v_j, v_{ij}) \E_{pop}\left[L_i^v L_j^v \right] \\
\approx& \sum_{v_i, v_j,v_{ij}} p(v_i, v_j, v_{ij}) \times \frac{1}{2np(v_i, v_{ij})} \times \frac{1}{2np(v_j, v_{ij})} \qquad\text{(from \eqref{eq:approx-E_pis_logpis})} \\
=& \sum_{v_i, v_j,v_{ij}} \frac{1}{4n^2} \times \frac{p(v_i, v_j, v_{ij})}{p(v_i, v_{ij})p(v_j,v_{ij})}  \numberthis \label{eq:parentShare}
\end{align*}

In the case where $X_j$ is a parent of $X_i$,
\begin{align*}
\E_{pop}[L_i L_j] =& \E_{pop}\left[\left(\sum_{v_i} \ind{i}{v_i} x_j L_i^v\right) \left(\sum_{v_j} \ind{j}{v_j} L_j^v\right)\right]\\
=& \E_{pop}\left[\sum_{v_i, v_j} \left(\ind{i}{v_i} \ind{j}{v_j} x_j L_i^v \left(x_j \log\frac{p_j^v}{\hat{p}_j^v} \right) \right) \right] \\
=& \sum_{v_i, v_j} p(v_i, v_j, x_{j}) \E_{pop}\left[L_i^v \log\frac{p_j^v}{\hat{p}_j^v} \right] \\
\approx& \sum_{v_i, v_j} p(v_i, v_j, x_{j}) \times \frac{1}{2n p(v_i, x_j)} \times \frac{1-p_j^v}{2np_j^v} \text{(from \eqref{eq:approx-E_pis_logpis})}  \\
=& \sum_{v_i, v_j} \frac{1-p_j^v}{4n^2} \times \frac{p(v_i, v_j, x_{j})} {p(v_i, x_j) p_j^v}  \numberthis \label{eq:Parent}
\end{align*}

% !TEX root = ../main.tex
\section{Generic Categorical Variable}\label{app:categorical}

In this section, we generalize our results to any categorical variables (not just binary). The extension from binary to categorical is straightforward. We will have a similar expression for the likelihood ratio statistic:
 
 \begin{align*}
L(x) =& \log\left(\frac{\Pr(x; \pgmpop)}{\Pr(x ; \pgmpool)}\right)\\
=& \sum_{i=1}^{m} L_i,
\end{align*}
where $L_i$ is the contribution of $X_i$ to $L(x)$:
\begin{align*}
L_i  =& \sum_{v} \ind{i}{v} L_i^v
\end{align*}

Instead of writing $L_i^v$ as
\begin{equation*}
L_i^v  = x_i \log \frac{p_i^v}{\hat{p}_i^v} + (1 - x_i) \log \frac{1 - p_i^v}{1 - \hat{p}_i^v}
\end{equation*}
we write
\begin{equation*}
L_i^v  = \sum_{o \in V(X_i)} 1_{\{x_i=o\}}\log \frac{p_{io}^v}{\hat{p}_{io}^v},
\end{equation*}
$$p_{io}^v = \Pr(x_i = o | Pa_{X_i}^G = v)$$

Now,

\begin{align*}
\E_{pop}[L_i^v] =& \sum_{o \in V(X_i)} E\left[p_{io}^v \log\frac{p_{io}^v}{\hat{p}_{io}^v}\right] \\
=& \sum_{o \in V(X_i)} \frac{1-p_{io}^v}{2n_i^v} \quad \text{(from \eqref{eq:approx_pis_log})}\\
=& \frac{|V(X_i)|-1}{2n_i^v} \numberthis \label{eq:generalOneTerm}
\end{align*}

\begin{align*}
\E_{pop}[L_i] =& \sum_{v} E[\ind{i}{v} L_i^v | Pa_{X_i}^G =v]\\
=& \sum_{v}E_{\hat{p}_i^v} \left[E_x\left[\ind{i}{v} L_i^v \mid \hat{p}_i^v\right] \right] \\
=& \sum_{v}p_{v} \frac{|V(X_i)|-1}{2n_i^v} \quad \text{(from \eqref{eq:generalOneTerm})}\\
=& \sum_{v} \frac{|V(X_i)|-1}{2n} + O(n^{-2}) \quad \text{(from \eqref{eq:approx-1_over_nis})}\\
=& |V(Pa_{X_i})| \times \frac{|V(X_i)|-1}{2n} + O(n^{-2})
\end{align*}

Now we can calculate $\E_{pop}[L(x)]$ as:
\begin{align*}
\E_{pop}[L(x)] =& \sum_{i=1}^mE_{pop}[L_i]\\
\approx& \sum_{i=1}^m|V(Pa_{X_i})| \times \frac{|V(X_i)|-1}{2n} + O(n^{-2})\\
=& \frac{C(G)}{2n} +  O(C(G)n^{-2})
\end{align*}

Hence,
\begin{align}
\E_{pop}[L(x)] =& \frac{C(G)}{2n}  + O(C(G)n^{-2})
\end{align}

Similarly for deriving variance we have,
\begin{align*}
\E_{pop}[(L_i^v)^2] =& \sum_{o \in V(X_i)} \E\left[p_{io}^v \left(\log\frac{p_{io}^v}{\hat{p}_{io}^v}\right)^2\right] \\
=& \sum_{o \in V(X_i)} \frac{1-p_{io}^v}{n_i^v} \quad \text{(from \eqref{eq:approx_square_log})} \\
=& \frac{|V(X_i)|-1}{n_i^v} \numberthis \label{eq:generalSquareTerm}
\end{align*}

Using equation \eqref{eq:generalSquareTerm}, we can calculate $\E_{pop}[L_i^2]$ as:

\begin{align*}
\E_{pop}[L_i^2] =& \sum_{v} \E[\ind{i}{v} (L_i^v)^2 | Pa_{X_i}^G = v]\\
=& \sum_{v} \E_{\hat{p}_i^v} \left[E_x\left[\ind{i}{v} (L_i^v)^2 \mid \hat{p}_i^v\right] \right] \\
=& \sum_{v}p_{v} \frac{|V(X_i)|-1}{n_i^v} \quad \text{(from \eqref{eq:generalSquareTerm})}\\
=& \sum_{v} \frac{|V(X_i)|-1}{n} + O(n^{-2}) \quad \text{(from \eqref{eq:approx-1_over_nis})}\\
=& |V(Pa_{X_i})| \times \frac{|V(X_i)|-1}{n} + O(n^{-2})
\end{align*}
Hence,
\begin{align*}
\sum_{i=1}^m \E_{pop}[L_i^2] =& \sum_{i=1}^m|V(Pa_{X_i})| \times \frac{|V(X_i)|-1}{n} + O(n^{-2})\\
=& \frac{C(G)}{n} +  O(C(G)n^{-2})
\end{align*}
\begin{align*}
\Var_{pop}(L) =& \E_{pop}[L^2]- (\E_{pop}[L])^2\\
\E_{pop}[L^2] =& \sum_{i =1}^m \E[L_i^2] + 2 \sum_{ 1 \leq i < j \leq m} \E[L_iL_j]
\end{align*}
Similar to the derivations of $\sum \E_{pop}[L_i]$ and $\sum \E_{pop}[L_i^2]$, we will have 

\begin{align*}
\sum_{i,j} \E_{pop}[L_iL_j] = \frac{(C(G))^2}{4n^2} + O((C(G))^2n^{-2})
\end{align*}

Hence, for categorical variables:
\begin{align}
Var_{pop}[L(x)] =& \frac{C(G)}{n}  + O((C(G))^2n^{-2})
\end{align}
% !TEX root = ../main.tex
\section{Naive Bayes}\label{app:naive-bayes}
 In section \ref{subsec:derivation-mean-variance}, while deriving the variance, we haven't calculated the exact value of the $O((C(G))^2 n^{-2})$ term. From the released model, it is possible to calculate the exact value of this term. Here we derive the exact value of the $O((C(G))^2 n^{-2})$ term, when the released model is a Naive Bayes model. Let the number of attributes in the model be equal to $m$. Hence, the complexity of the model is $C(G) = 2m - 1$. Let $X_1$ be the class variable and $p_i^1 = Pr(X_i =1 |X_1 = 1)$. Then, using equation \eqref{eq:approx-E_pis_logpis} we have:
	\begin{align*}
	\E_{pop}(L)
	&= \E_{pop}\left[ x_1 \log \frac{p_1}{\hat{p}_1} + (1 - x_1) \log \frac{1 - p_1}{1 - \hat{p}_1} + x_1 \sum_{i=2}^{m} \left( x_i \log \frac{p_i^1}{\hat{p}_i^1} + (1 - x_i) \log \frac{1 - p_i^1}{1 - \hat{p}_i^1} \right) \right. \\
	&\left. \quad \quad \quad \quad \quad \quad \quad \quad \quad \quad \quad \quad \quad \quad \quad \quad \quad + (1-x_1)\sum_{i=2}^{m} \left(x_i\log\frac{p_i^0}{\hat{p}_i^0}+(1-x_i)\log\frac{1-p_i^0}{1-\hat{p}_i^0}\right) \right] \\
	&= \frac{1}{2n}+ \sum_{i=2}^{m} \left[ p_1 \times \frac{1}{2np_1}+ \frac{1}{2n^2} \left[ \frac{1-p_1}{p_1}  \right] \right] +  \sum_{i=2}^{m} \left[ \left( 1-p_1 \right) \times \frac{1}{2n(1-p_1)} + \frac{1}{2n^2} \left[ \frac{p_1}{1-p_1}  \right] \right] \\
	&= \frac{2m-1}{2n} + O(mn^{-2}) \\
	&= \frac{C(G)}{2n} + O(C(G)n^{-2})  \numberthis \label{eq:NaiveBayesMean}
	\end{align*}

We can calculate the exact value of $\E_{pop}(L^2)$ using the equations \eqref{eq:approx-E_pis_logpis_sq}, \eqref{eq:parentShare} and \eqref{eq:Parent} as below :
	\begin{align*}
	\E_{pop}(L^2)
	&= \E_{pop}\left[ \left[ x_1 \log \frac{p_1}{\hat{p}_1} + (1 - x_1) \log \frac{1 - p_1}{1 - \hat{p}_1} + x_1 \sum_{i=2}^{m} \left( x_i \log \frac{p_i^1}{\hat{p}_i^1} + (1 - x_i) \log \frac{1 - p_i^1}{1 - \hat{p}_i^1} \right) \right. \right.\\
	& \quad \quad \quad \quad \quad \quad \quad \quad \quad \quad \quad \quad \quad \quad \quad  + (1-x_1)\sum_{i=2}^{m} \left(x_i\log\frac{p_i^0}{\hat{p}_i^0} \left. \left. +(1-x_i)\log\frac{1-p_i^0}{1-\hat{p}_i^0}\right) \right]^2 \right] \\
	&\stackrel{}{=} \frac{1}{n}+ \sum_{i=2}^{m} \left[ p_1 \times\frac{1}{np_1} + \frac{1}{n^2} \left[ \frac{1-p_1}{p_1} \right] \right]  + \sum_{i=2}^{m} \left[ \left( 1-p_1 \right) \times \frac{1}{n(1-p_1)} + \frac{1}{n^2} \left[ \frac{p_1}{1-p_1}  \right] \right] \\
	& \quad  + 2 \left[ {m-1 \choose 2} \times \frac{p_1}{4n^2\hat{p}_1^2} + {m-1 \choose 2} \times \frac{1-p_1}{4n^2(1-\hat{p}_1)^2} + \frac{(m-1)(1-p_1)}{4n^2} + \frac{(m-1)(p_1)}{4n^2} \right] \\
	&= \frac{2m-1}{n} + \frac{(m-1)(m-2)}{4n^2} \left[ \frac{p_1}{\hat{p}_1^2} + \frac{1-p_1}{(1-\hat{p}_1)^2} \right] + O(mn^{-2}) \\
	&\approx \frac{C(G)}{n} + \frac{m^2}{4n^2} \left[\frac{1}{p_1(1-p_1)} \right] + O(mn^{-2})  \numberthis \label{eq:NaiveBayesSquare}
	\end{align*}
	
	Combining equations \eqref{eq:NaiveBayesMean} and \eqref{eq:NaiveBayesSquare}, we have the variance for Naive Bayes as:
	
	\begin{align*}
	\Var_{pop}(L)
	&= \E_{pop}(L^2) - (\E_{pop}(L))^2  \\
	&= \frac{C(G)}{n} + \frac{m^2}{4n^2} \left[\frac{1}{p_1(1-p_1)} -4 \right] + O(mn^{-2})  \\
	&\approx \frac{C(G)}{n} + O((C(G))^2n^{-2}) \numberthis \label{eq:NaiveBayesVariance}
	\end{align*}
	
% !TEX root = ../main.tex
\section{Understanding the complexity metric - Parameter estimation errors} \label{sec:fisher}
% 1. We need to connect the definition of complexity

% 2. Number of parameters

% 3. Each parameter is estimated from data.

% 4. How much information does a point have about the parameters?

% 5. 

The complexity of a Bayesian network $\pgmpop$ with discrete random variables is the number of independent parameters used to define its probability distribution.
\begin{align*}
	C(G) = \sum_{i=1}^m |V(Pa_{X_i}^G)| (|V(X_i)| - 1)
\end{align*}
	The parameters $\theta$ are estimated from the pool data. To understand the privacy risk of this learning to members of the pool, we need to study the influence a member can have on the value of the parameters. Fisher information quantifies the amount of information a random variable carries about the parameter(s) $\theta$ of the probability distribution from which it is generated.
	\begin{equation*}
		I(\theta) = -\E_\theta(\nabla^2 l(\theta)),
	\end{equation*}
	% $$I(\theta) = -\E_\theta(\nabla^2 l(\theta)),$$
	where $I(\theta)$ is Fisher information, and $l(\theta)$ is the log-likelihood function for $\theta$.

If $\hat{\theta}$ is a Maximum Likelihood Estimate of $\theta$, then it is known that
 	\begin{equation*}
 		\hat{\theta} = Normal(\theta, I(\hat{\theta})^{-1}).
 	\end{equation*}
	% $$\hat{\theta} = Normal(\theta, I(\hat{\theta})^{-1}).$$

The log-likelihood functions of parameter $\theta$ from a PGM $\pgmpop$, given a sample $x$ are typically of the form:
	\begin{align*}
	l(\theta) =& \log\left[\Pr(x;\pgmpop)\right]\\
		=& \sum_{i=1}^{m} l_i
	\end{align*}
where $l_i$ is contribution of $X_i$ to the likelihood function:
	\begin{align*}
	l_i =& \sum_{ v \in V(Pa_{X_i}^G)} \ind{i}{v} l_i^v \\
	l_i^v	=& \sum_{o \in V(X_i)} 1_{\{x_i= o\}}\log p_{io}^v
	\end{align*}

	\begin{equation*}
		l = \sum_{i, v, o} f_{i, v, o} (x_1, x_2, \ldots, x_m) \log(p_{io}^v),
	\end{equation*}
	% $$l = \sum_{i, v, o} f_{i, v, o} (x_1, x_2, \ldots, x_m) \log(p_{io}^v),$$
	where $f_{i, v, o}$ are activator functions (some combination of $x_i$'s) for the parameter $p_{io}^v$.
	\begin{equation*}
		I(p) = -E_p(\nabla^2 l(p))
	\end{equation*}
	% $$I(p) = -E_p(\nabla^2 l(p))$$

	All the non-diagonal elements of the information matrix are zero, because:
	\begin{equation*}
	\frac{\partial}{\partial p_{io}^v}\frac{\partial}{\partial p_{jo}^v} [\sum_{i, v, o} f_{i, v, o} (x_1, x_2, \ldots ,x_m)\log(p_{io}^v)] = 0, \forall  i \neq j
	\end{equation*}
	This implies that all the standard normal variables used to represent frequencies in pool are pair-wise independent i.e. all the estimation errors are independent across parameters. The difference between an estimated parameter value (calculated from the pool) and the actual parameter value (calculated from the general population) is the estimation error that leaks information about the pool. Since these estimation errors are independent across parameters of a Bayesian network, each parameter makes a separate contribution to the power of the attacker. Hence, the complexity measure defined as the number of independent parameters, captures the potential privacy risk of the model.

\section{\textbf{What about models trained with differential privacy?}} \label{sec:discussion-dp}
The bound provided in Theorem 1 is computed assuming that the parameters are learned without any privacy defense. The parameters can also be learned with a privacy defense (like differential privacy) in place. The effect of a differentially private learning mechanism on our bound can be better reasoned under the recently introduced notion of ``f-differential privacy'' ($f$-DP) \cite{dong2019gaussian}. $f$-DP is a new relaxation of differential privacy based on a framework of hypothesis testing. It characterizes the trade-off between type I and type II errors in distinguishing any two neighboring datasets using a function $f$. When the function $f$ is from a specific family that characterizes the trade-off between type I and type II errors in distinguishing the two normal distributions $\mathcal{N}(0,1)$ and $\mathcal{N}(\mu,1)$ based on one draw, it is said to be $\mu$-GDP. If the learning mechanism satisfies $\mu$-GDP, then the bound on power of membership inference in Theorem 1 will become:  
\begin{align}\label{eq:dp-effect}
z_{\alpha}+z_{1-\beta} \leq \mu 
\end{align}
Corollary 2.13 in the paper \cite{dong2019gaussian} provides the relationship between $\mu$-GDP and the standard $(\epsilon,\delta)-DP$.

\textbf{Corollary 1 ~\cite{dong2019gaussian}:}
    A mechanism is $\mu$-GDP if and only if it is $\big(\epsilon,\delta(\epsilon)\big)$-DP for all $\epsilon \geq 0$, where
    \[
    \delta(\epsilon)= \Phi\Big( -\frac{\epsilon}{\mu} +\frac{\mu}{2} \Big)-
    \mathrm{e}^{\epsilon}\Phi\Big(- \frac{\epsilon}{\mu} - \frac{\mu}{2} \Big)
    \], and $\Phi$ is the CDF of standard normal distribution.

Using Corollary 1 and equation~\ref{eq:dp-effect}, we can calculate how our bound in Theorem 1 changes when the parameters are learned with differential privacy guarantees.

\section{Evaluation details: Bayesian network learning and Data synthesis}\label{sec:pgm-appendix}

In this section, we describe the methods used in the evaluation for learning structure and parameters of a Bayesian network and generating synthetic data. See~\cite{koller2009probabilistic} for a comprehensive overview of the methods to learn Bayesian networks.
% Given a multi-dimensional dataset, the training algorithm for graphical models involve learning the structure of the model as well as its parameters.  Ideally, both the structure learning and parameter learning need to be done using a joint optimization that maximizes the likelihood of data points in the training set. However, due to its high computational complexity, they are optimized in sequence.
 
\subsection{Structure Learning} \label{sec:pgm:structure}

The objective is to learn the significant dependencies between random variables, and represent them as a graph.  We used an existing algorithm based on maximizing a score function that measures how correlated different attributes are, according to the training data~\cite{hall1999correlation}.  For each attribute we find a set of attributes which are highly correlated with it, yet are not significantly correlated among themselves.  
\begin{align}
	score(Pa_{X_i}^G) = \frac{\sum_{X_j \in Pa_{X_i}^G} corr(X_i, X_j)}{\sqrt{|Pa_{X_i}^G|+ \sum_{x_j,x_k \in Pa_{X_i}^G}corr(X_j, X_k)}},
\end{align}
\begin{align*}
	corr(X_i,X_j) = 2- 2\frac{H(X_i,X_j)}{H(X_i)+H(X_j)},
\end{align*}
where $H$ is the entropy function.

While optimizing this score for each attribute, we need to make sure that the graph remains acyclic.  Also, to control the complexity of the graph, we impose a condition on $\eta$, the maximum number of parents for each node.  We use an iterative and greedy algorithm that adds parents to each node while maximizing the score for all nodes at each iteration, subject to the constraints. 

% Note that in practice structure learning might involve expert knowledge, and is not always totally dependent on the training data.  It could also be the result of the consensus of a community (e.g., genomicists) on the correlation and dependency between attributes of a particular type of data.
\subsection{Parameter Learning} \label{sec:pgm:parameters}

We assume a prior distribution on all possible values of the parameters $\theta$, and use the training data set to update this distribution, using a Bayesian approach. 

Let $X_i$ be the random variable for a categorical attribute.  Let $\vec{\theta}_i$ be the parameters of the conditional probability $\Pr[X_i | \parpop]$.  For each assignment of values to $\parpop$, we assume a prior distribution on all the possible $k$-dimensional multinomial distributions.  The prior distribution for each assignment $v$ comes from a Dirichlet family, i.e., $\vec{\theta}_i^v \sim Dirichlet(\vec{\alpha}_i^v)$, where $\vec{\alpha}_i^v$ is the hyper parameters of the distribution. 

Let $\vec{c}_i^v = [c_{i1}^v, c_{i2}^v, \cdots, c_{ik}^v]$ include the frequency of the events $[X_i = j | \parpop = v]$ in training data.  We compute the posterior distribution for $\vec{\theta}_i^v$ as $Dirichlet(\vec{\alpha}_i^v + \vec{c}_i^v)$.  Thus, the most likely estimation for set of parameters $\vec{\theta}_i^v$ is:
\begin{align}\label{eq:parestimation}
	\theta_{ij}^v = \frac{\alpha_{ij}^v + c_{ij}^v}{\sum_{j=1}^k (\alpha_{ij}^v + c_{ij}^v)}.
\end{align}
In all our experiments, we use a uniform prior i.e., we set $\vec{\alpha}_i^v$ to $1$ in all dimensions.

\subsection{Data Synthesis} \label{sec:pgm:synthesis}

Given a data set $D$, we want to synthesize datasets that are close in distribution to $D$. Graphical models could be used for inference and prediction, as well as generating synthetic data (from the underlying distribution that they encode).  We use the below process for generating synthetic datasets:

\begin{enumerate}
  \item Learn a Bayesian network $\pgmpop$ from the data set $D$.
  \item Create a Bayesian network $\pgmpopp$ with $G' = G$, and $\theta'$ drawn from the posterior Dirichlet distribution for $\theta$, which was computed during parameter learning. 
  \item Draw independent samples from $\pgmpopp$.
\end{enumerate}

In our experiments, while generating the synthetic data, we use $\eta =3$ for learning the structure $G$ of the Bayesian network $\pgmpop$ from the data set $D$.
% \clearpage
% \newpage
% !TEX root = ../main.tex
\section{Additional evaluation}
\subsection{Effect of releasing statistically insignificant edges}

We analyze the effect on the power of attack of releasing edges (conditional probabilities) that are statistically insignificant. To perform this evaluation, we consider the case where the structure of released model is not learned from data but generated in a random way. 

Figure~\ref{table: RandomEdges} compares the power of the attack when two different models of almost equal complexity are released. The structure of first model is generated by randomly adding edges and the structure of second model is learned from data. Adversary uses the released model as the population model to calculate the LR statistic and perform tracing attack. We can observe that the attack power is similar in both the cases.

The edges that are generated randomly might not be statistically significant, but they leak about membership. In the LR Test for tracing attack, we rely on the difference between probability distributions for pool and population. Adding a statistically insignificant edge gives similar probability distributions for all configurations of the parent. Although the conditional probabilities are similar, their values will be different for pool and population and hence they will leak about membership.  \textbf{Statistically insignificant edges leak as much information about membership as significant edges.}

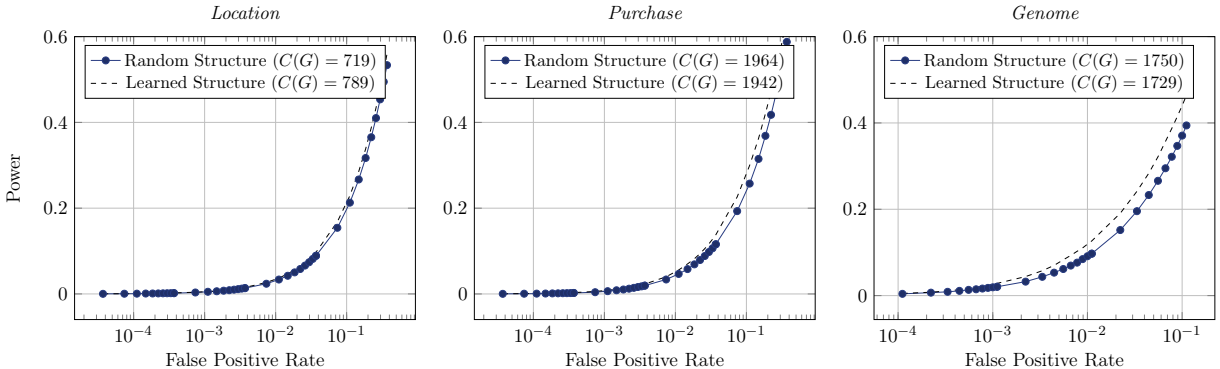
\begin{figure*}[htb]
	\centering
	\resizebox{1\columnwidth}{!}{%
		\begin{tabular}{llll}
			
			\begin{tikzpicture}
			\begin{semilogxaxis}
			[legend style={font=\small},title=  {\em Location} ,xlabel={False Positive Rate},ylabel={Power}, ymax=0.6, grid = major, legend entries = {
				Random Structure $(C(G) = 719) $,
				Learned Structure $(C(G)= 789)$
			},
			legend pos =  north west]
			\addplot table[x index=0,y index=1] {"Data/Bangkok/Random/Mean_N_Parents1.dat"};
			\addplot[dashed] table[x index=0,y index=1] {"Data/Bangkok/Mean_N_Parents1.dat"};

			\end{semilogxaxis}
			\end{tikzpicture} &
			
			\begin{tikzpicture}
			\begin{semilogxaxis}
			[legend style={font=\small},title=  {\em Purchase} ,xlabel={False Positive Rate},ylabel={}, ymax=0.6, grid = major, legend entries = {
				Random Structure $(C(G) = 1964)$,
				Learned Structure $(C(G)= 1942)$
			},
			legend pos =  north west]

			\addplot table[x index=0,y index=1] {"Data/Purchase/Random/Mean_N_Parents2.dat"};
			\addplot[dashed] table[x index=0,y index=1] {"Data/Purchase/Mean_N_Parents2.dat"};

			\end{semilogxaxis}
			\end{tikzpicture}   &
			
			\begin{tikzpicture}
			\begin{semilogxaxis}
			[legend style={font=\small},title=  {\em Genome} ,xlabel={False Positive Rate},ylabel={}, ymax=0.6, grid = major, legend entries = {
				Random Structure $(C(G)= 1750 )$,
				Learned Structure $(C(G)= 1729)$
			},
			legend pos =  north west]
			\addplot table[x index=0,y index=1] {"Data/Genome_1000/Random/Mean_N_Parents1.dat"};
			\addplot[dashed] table[x index=0,y index=1] {"Data/Genome_1000/Mean_N_Parents1.dat"};
	
			\end{semilogxaxis}
			\end{tikzpicture}   \\
			
		\end{tabular}
	}
	\caption{\small{\textbf{Effect of Releasing Graphical Models with Random Edges:} Here we compare the power of attack, when two models of similar complexity learned on the dataset are released but structure of one model is learned from data and the structure of other is generated randomly. We can see that for close values of $C$, the power of attack is almost same for both the models. This shows that statistically insignificant edges leak as much information about membership, as that of significant edges. }}
	\label{table: RandomEdges}
\end{figure*}
% \clearpage
% \newpage
% !TEX root = ../main.tex
\subsection{Optimality of the theoretical threshold}

In Figure~\ref{table: thresholdPlot}, we compare theoretical thresholds for certain false positive rates with their corresponding values estimated using the reference population. The adversary has access to some reference population. For a given false positive rate, the adversary chooses the threshold based on the likelihood ratio on the reference population data. The attacker then runs the LR test tracing attack.  When $\eta = 0$ (row 1), we observe that the theoretical threshold values are much higher than the estimated values. When $\eta = 3$ (row 2), the observed thresholds are closer to the estimated values.

When $\eta = 0$, the parameter estimation errors are correlated, which reduces the amount of information leakage. The adversary, when using the theoretical threshold, overestimates the amount of leakage (power) and hence chooses a higher threshold. When $\eta = 3$, the released model captures most of the dependencies among attributes in the data. Hence the observed threshold will be closer to the theoretical threshold values.  \textbf{From the adversary's perspective, the theoretical threshold value is sub-optimal when the released model is underfitted (loss of utility).}

\begin{figure*}[htb]
	\centering
	\resizebox{1\columnwidth}{!}{%
		\begin{tabular}{llll}
			
			\begin{tikzpicture}
			\begin{semilogxaxis}
			[legend style={font=\small},title= {\em Location} ,xlabel={},ylabel={Threshold}, grid = major, legend entries = {Attack , Theoretical}, ymin= -6, ymax=0, legend pos =  south east]
			\addplot table[x index=0,y index=2] {"Data/Bangkok/Adv_Mean_N_Parents0.dat"};
			\addplot table[x index=0,y index=2] {"Data/Bangkok/exact_theory_threshold0.dat"};
			\end{semilogxaxis}
			\end{tikzpicture} &
			
			\begin{tikzpicture}
			\begin{semilogxaxis}
			[legend style={font=\small},title= {\em Purchase} ,xlabel={},ylabel={}, grid = major, legend entries = {Attack , Theoretical},ymin= -6, ymax=0, legend pos =  south east]
			\addplot table[x index=0,y index=2] {"Data/Purchase/Adv_Mean_N_Parents0.dat"};
			\addplot table[x index=0,y index=2] {"Data/Purchase/exact_theory_threshold0.dat"};
			\end{semilogxaxis}
			\end{tikzpicture} &
			
			\begin{tikzpicture}
			\begin{semilogxaxis}
			[legend style={font=\small},title= {\em Genome} ,xlabel={},ylabel={},ylabel right ={$\eta =0$}, grid = major, legend entries = {Attack, Theoretical }, ymin= -6,ymax=0, legend pos =  north west]
			\addplot table[x index=0,y index=2] {"Data/Genome_1000/Adv_Mean_N_Parents0.dat"};
			\addplot table[x index=0,y index=2] {"Data/Genome_1000/exact_theory_threshold0.dat"};
			\end{semilogxaxis}
			\end{tikzpicture} &

			\\
			
			\begin{tikzpicture}
			\begin{semilogxaxis}
			[legend style={font=\small},title= {} ,xlabel={False Positive Rate},ylabel={Threshold}, grid = major, legend entries = {Attack, Theoretical  }, ymin= -6, ymax=0, legend pos =  north west]
			\addplot table[x index=0,y index=2] {"Data/Bangkok/Adv_Mean_N_Parents3.dat"};
			\addplot table[x index=0,y index=2] {"Data/Bangkok/exact_theory_threshold3.dat"};
			\end{semilogxaxis}
			\end{tikzpicture} &
			
			\begin{tikzpicture}
			\begin{semilogxaxis}
			[legend style={font=\small},title= {} ,xlabel={False Positive Rate},ylabel={}, grid = major, legend entries = {Attack , Theoretical  }, ymin= -6,ymax=0, legend pos =  north west]
			\addplot table[x index=0,y index=2] {"Data/Purchase/Adv_Mean_N_Parents3.dat"};
			\addplot table[x index=0,y index=2] {"Data/Purchase/exact_theory_threshold3.dat"};
			\end{semilogxaxis}
			\end{tikzpicture} &
			
			\begin{tikzpicture}
			\begin{semilogxaxis}
			[legend style={font=\small},title= {} ,xlabel={False Positive Rate},ylabel={},ylabel right ={$\eta =3$}, grid = major, legend entries = {Attack, Theoretical  }, ymin = -6,ymax=0, legend pos =  north west]
			\addplot table[x index=0,y index=2] {"Data/Genome_1000/Adv_Mean_N_Parents3.dat"};
			\addplot table[x index=0,y index=2] {"Data/Genome_1000/exact_theory_threshold3.dat"};
			\end{semilogxaxis}
			\end{tikzpicture} &

			\\
		\end{tabular}
	}
	\caption{\small{\textbf{Effect of releasing underfitted models on threshold selection:} This plot compares the threshold values estimated by the adversary using reference population at different false positive rates with their corresponding theoretical values. The label Attack indicates that the threshold is estimated by the adversary using reference population. We observe that for underfit models ($\eta = 0$) (first row), the threshold value estimated from the reference population is way less than the corresponding theoretical value. As the model gets closer to the generator distribution ($\eta = 3$) (second row), the estimated threshold values get closer to the theoretical values. 
	}}
	\label{table: thresholdPlot}
\end{figure*}
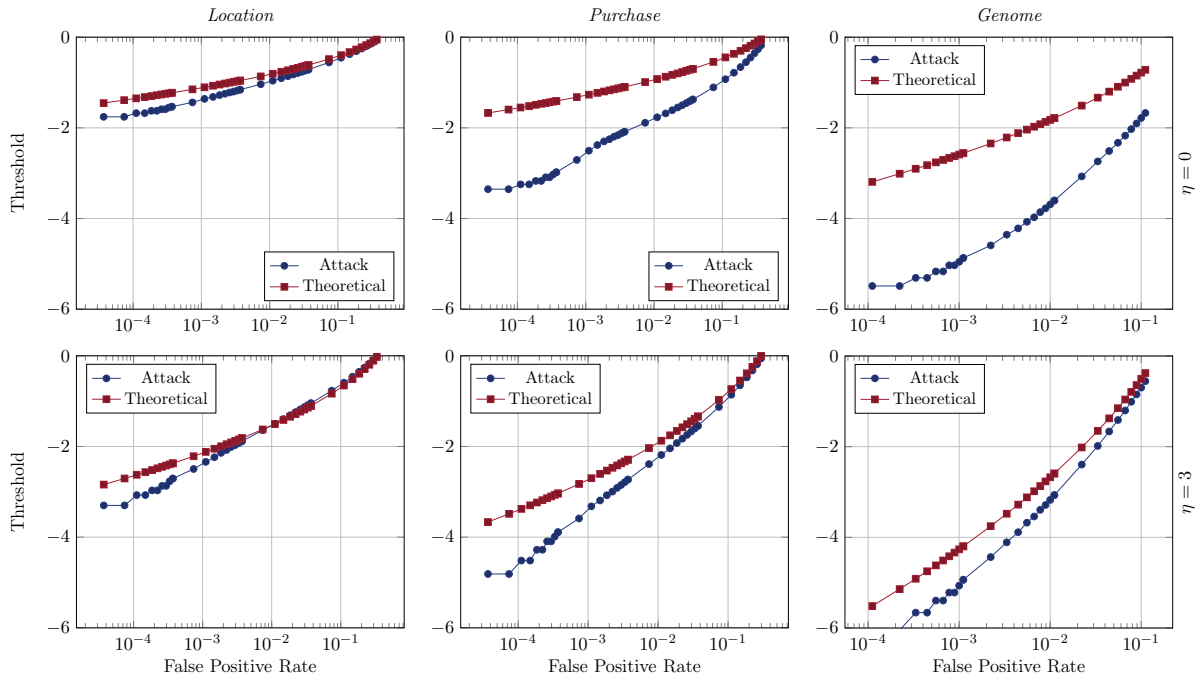
% \clearpage
% \newpage
% !TEX root = ../main.tex
\subsection{Effect of using a complex model for estimating likelihood of Null hypothesis}

In this subsection, we present the effect of population model choice on the behavior of the likelihood ratio test and on the power of the tracing attack. Specifically, we study the effect of using models that are more complex than the released model as population model. 
The parameters of a graphical model $\pgmpool$ with $\eta = 1$ are learned on the pool data and released. The adversary has access to a complex and better representative model $\pgmdifpop$ that was learned with $\eta = 3$. The adversary can choose to use either the released model structure $G$ or a complex model structure $G_{pop}$ as population model structure.

Figure~\ref{fig:lrdistriPoolPop} compares the empirical distribution of test statistic (likelihood ratio) values computed on members of the pool and on non-members for both choices of population model on the genome dataset. On the right, we observe that the member distribution is indistinguishable from the non-member distribution when $G_{pop}$ (learned with $\eta =3$) is used as structure of population model. 
We also observe that the values of the likelihood ratio are much higher -- from 20 to 70 -- compared to the values we observe on the left (narrowly concentrated around 0) when the structure of population model is same as that of released model (learned with $\eta=1$).

When a complex model is used as population model, the likelihood value of the null hypothesis increases for both members and non-members. Hence it cannot help in distinguishing members from non-members. \textit{Also it changes the meaning of the hypothesis test. When a complex model is used to compute the likelihood of null hypothesis, the computed likelihood is no longer the likelihood of the target being a random sample from the population. The meaning of this new hypothesis test would be the following: which of the \textbf{models is more likely} given the target. Since complex models are more likely compared to simpler models, the test statistic (likelihood ratio) values will be very high and positive.} Figure~\ref{fig:choosingG} compares the power of the tracing attack for both choices of population model. The power of the attack is higher when the released model structure is used as the population model structure. \textbf{Knowledge of additional statistics about population other than the released statistics doesn't increase the power of adversary.}

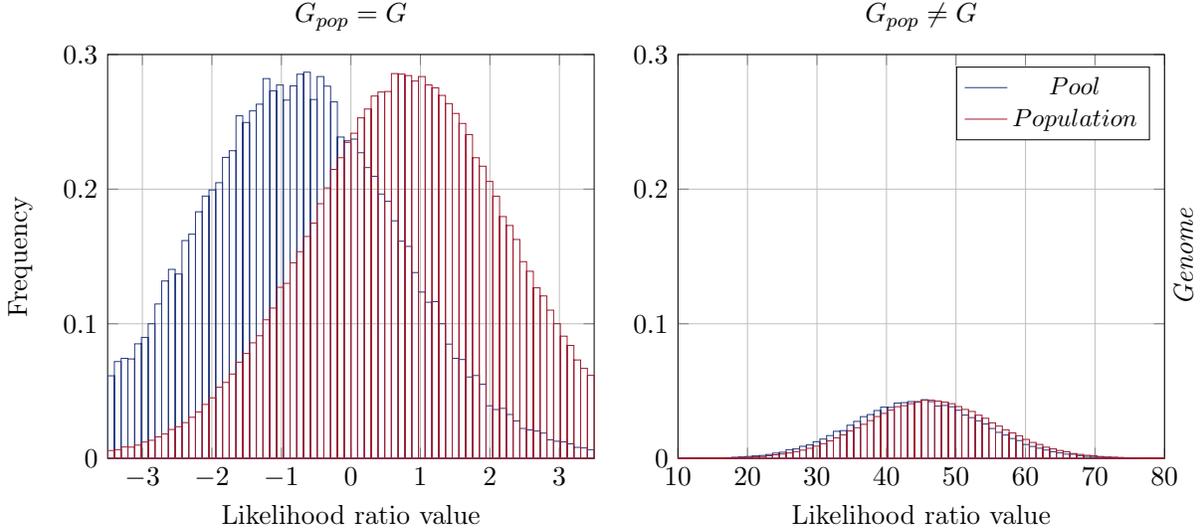
\begin{figure*}[htb]
	\centering
	\resizebox{1\columnwidth}{!}{%
		\begin{tabular}{lll}
			
			\begin{tikzpicture}
			\begin{axis}
			[xmin=-3.5, xmax=3.5, ymin=0, ymax=0.3,legend style={font=\small},title=  {\em $G_{pop} = G$} ,xlabel={Likelihood ratio value},ylabel={Frequency}, grid = major]
			\addplot+[ybar interval,mark=no] table [x index=0, y index =1] {"Data/LRvalues_Hist/hist_pool_1_1.dat"};
			\addplot+[ybar interval,mark=no] table [x index=0, y index =1] {"Data/LRvalues_Hist/hist_pop_1_1.dat"};
			
			\end{axis}
			\end{tikzpicture} &

			\begin{tikzpicture}
			\begin{axis}
			[xmin=10, xmax=80, ymin=0, ymax=0.3,legend style={font=\small},title=  {\em $G_{pop} \neq G$} ,xlabel={Likelihood ratio value}, ylabel right ={\em Genome}, grid = major, legend entries = {$Pool $, $Population$ },  legend pos =  north east]
			\addplot+[ybar interval,mark=no] table [x index=0, y index =1] {"Data/LRvalues_Hist/hist_pool_1_3.dat"};
			\addplot+[ybar interval,mark=no] table [x index=0, y index =1] {"Data/LRvalues_Hist/hist_pop_1_3.dat"};

			\end{axis}
			\end{tikzpicture} \\
			
		\end{tabular}
	}
	\caption{\small{\textbf{Comparison of likelihood ratio distributions computed on members of the pool (blue histogram) and on non-members (red histogram):} 
	 \textbf{Left:} To calculate the likelihood of the null hypothesis $\hout$, we use the population model $\pgmpop$, whose structure is the same as that of the released model $\pgmpool$ ($\eta = 1$). We observe that the member distribution is clearly distinguishable from the non-member distribution.
	\textbf{Right:} To calculate the likelihood of the null hypothesis $\hout$, we use the population model $\pgmdifpop$, whose structure is different and more complex ($\eta = 3$) than the released model $\pgmpool$ ($\eta = 1$). We observe that the member/non-member distributions are indistinguishable. Also, the values of the likelihood ratio are much higher compared to the left part of the figure.
	 Using a complex population model might increase the likelihood of null hypothesis $\hout$, but it increases the value for both members and non-members (as a complex model can explain both members and non-members better than a simple model can), making them indistinguishable.
	Hence the optimal choice of population model for the adversary is the released model estimated over the reference population.
	}}
	\label{fig:lrdistriPoolPop}
\end{figure*}

\begin{figure*}[htb]
	\centering
	\caption{\small{\textbf{Effect of using a complex model as population model:} The parameters of a graphical model $\pgmpool$ with $\eta =1$ are learned on the pool data and the model is released. The adversary has access to a better (more complex) generative model $\pgmdifpop$ with ($\eta =3$). We observe how the power of the attack changes when calculating the likelihood of null hypothesis using this complex generative model structure instead of the released model structure. We can see that the power of the attack reduces when the population model structure is not the same as the released model structure. As shown in Figure~\ref{fig:lrdistriPoolPop}, using a complex population model increases the likelihood for both members and non-members and hence cannot help in distinguishing them. This shows that it is not possible to increase the power of adversary using knowledge about additional statistics on the data that are not present in the released graphical model. }}
	\resizebox{1\columnwidth}{!}{%
		\begin{tabular}{llll}
			
			\begin{tikzpicture}
			\begin{semilogxaxis}
			[legend style={font=\small},title=  {\em Location} ,xlabel={False Positive Rate},ylabel={Power}, grid = major, legend entries = {$ C(G) = 789 \quad G_{pop} = G $ ,$C(G_{pop})= 1905 \quad G_{pop} \neq G $}, ymax=0.6, legend pos =  north west]
			\addplot table[x index=0,y index=1] {"Data/Bangkok/Mean_N_Parents1.dat"};
			\addplot table[x index=0,y index=1] {"Data/Bangkok/Mean_N_Parents13.dat"};
			
			\end{semilogxaxis}
			\end{tikzpicture} &
			
			\begin{tikzpicture}
			\begin{semilogxaxis}
			[legend style={font=\small},title=  {\em Purchase} ,xlabel={False Positive Rate},ylabel={}, grid = major, legend entries = {$C(G) = 1096 \quad G_{pop} = G$, $C(G_{pop})=3431 \quad G_{pop} \neq G $ }, ymax=0.6, legend pos =  north west]
			\addplot table[x index=0,y index=1] {"Data/Purchase/Mean_N_Parents1.dat"};
			\addplot table[x index=0,y index=1] {"Data/Purchase/Mean_N_Parents13.dat"};

			\end{semilogxaxis}
			\end{tikzpicture}    &
			
			\begin{tikzpicture}
			\begin{semilogxaxis}
			[legend style={font=\small},title=  {\em Genome} ,xlabel={False Positive Rate},ylabel={}, grid = major, legend entries = {$C(G)= 1729 \quad G_{pop} = G$, $C(G_{pop})=4323 \quad G_{pop} \neq G $ }, ymax=0.6, legend pos =  north west]
			\addplot table[x index=0,y index=1] {"Data/Genome_1000/Mean_N_Parents1.dat"};
			\addplot table[x index=0,y index=1] {"Data/Genome_1000/Mean_N_Parents13.dat"};
			\end{semilogxaxis}
			\end{tikzpicture} \\
			
		\end{tabular}
	}
	
	\label{fig:choosingG}
\end{figure*}
% \clearpage
% \newpage
% !TEX root = ../main.tex

\subsection{Effect of Biased Sampling}

In this section, we empirically study the effect of sampling bias on the power of tracing attack. We model a case of sampling bias, where we discriminate against individuals with some attribute value (say 1). We add a bias in the sampling mechanism for pool, by making the probability of selecting an attribute with value 1 as $1-bias$.
    \begin{align}
    Pr(select|X_i = 0) = 1
    \end{align}
    \begin{align}
    Pr(select|X_i = 1) = 1-bias
    \end{align}
% To evaluate The number of attributes in generated data is 1000, pool size 2000 and population size 10000. Parameters of the dirichlet distribution are learned from conditional probability values of graph constructed on OpenSNP1000 data (for the corresponding value of $\eta$). The effect of increase in power due to bias is shown in figure \ref{table: NewSamplingBiasPower}. We can clearly observe that, with increase in sampling bias, power of attack increases.
Synthetic data for this experiment was generated from graphs learned on Genome data. Pool is sampled in a biased way as described above. The parameter $bias$ can be used alter the amount of sampling bias. We generate a total of 10000 samples, of which we randomly select 2000 as pool and 4000 as reference population. 
   
Figure \ref{table: NewSamplingBiasPower} shows the effect of bias on power of attack. We can clearly observe that power of attack increases with increase in bias. When the pool is drawn from same distribution as population, we leveraged on finite sample estimation error for membership inference. If pool is drawn from distribution that is even slightly different from population distribution, power of attack increases and can be greater than provided bounds. \textbf{Biased sampling increases the power of tracing attack}

    \begin{figure*}[htb]
    \centering
    \resizebox{1\columnwidth}{!}{%
    \begin{tabular}{lll}

    \begin{tikzpicture}
        \begin{semilogxaxis}
        [legend style={font=\tiny},title=  {Biased Sampling$(\eta = 0)$} ,xlabel={False Positive Rate},ylabel={Power}, grid = major, legend entries = {$bias = 0.3$,$bias = 0.1$, $bias= 0(Theory)$ }, legend pos =  south east]
        \addplot table[x index=0,y index=1] {"Data/Final_Bias/01/Mean_N_Parents0.dat"};
        \addplot table[x index=0,y index=1] {"Data/Final_Bias/00/Mean_N_Parents0.dat"};

        %\addplot table[x index=0,y index=1] {"Data/Final_Bias/02/Mean_N_Parents0.dat"};
        \addplot[dashed, red] table[x index=0,y index=1] {"Data/Final_Bias/01/theory_threshold.dat"};
        \end{semilogxaxis}
    \end{tikzpicture} &

    \begin{tikzpicture}
        \begin{semilogxaxis}
        [legend style={font=\tiny},title=  {Biased Sampling$(\eta = 3)$} ,xlabel={False Positive Rate},ylabel={Power}, grid = major, legend entries = {$bias = 0.3$,$bias = 0.1$,$bias = 0(Theory)$ }, legend pos =  south east]
        \addplot table[x index=0,y index=1] {"Data/Final_Bias/31/Mean_N_Parents3.dat"};
        \addplot table[x index=0,y index=1] {"Data/Final_Bias/30/Mean_N_Parents3.dat"};
        %\addplot table[x index=0,y index=1] {"Data/Final_Bias/32/Mean_N_Parents2514.dat"};
        \addplot[dashed, red] table[x index=0,y index=1] {"Data/Final_Bias/31/theory_threshold.dat"};
        \end{semilogxaxis}
    \end{tikzpicture}    \\

    \end{tabular}
    }
    \caption{\textbf{Comparison of Power values in case of biased sampling:} Figure shows the effect of sampling bias on the power of tracing attack. We generate synthetic data using graph structures of different complexity that are learned on Genome data. The conditional probability values are generated from a Dirichlet distribution fitted to the conditional probabilities in corresponding graph of Genome data. The parameter $bias$ is used to tune the bias in sampling of the pool. We can observe that power of attack in case of biased sampling is greater than the theoretical bound (with out considering bias). With increasing value of $bias$, the pool distribution deviates more from the population distribution, which increases the power of attack. }
    \label{table: NewSamplingBiasPower}
    \end{figure*}
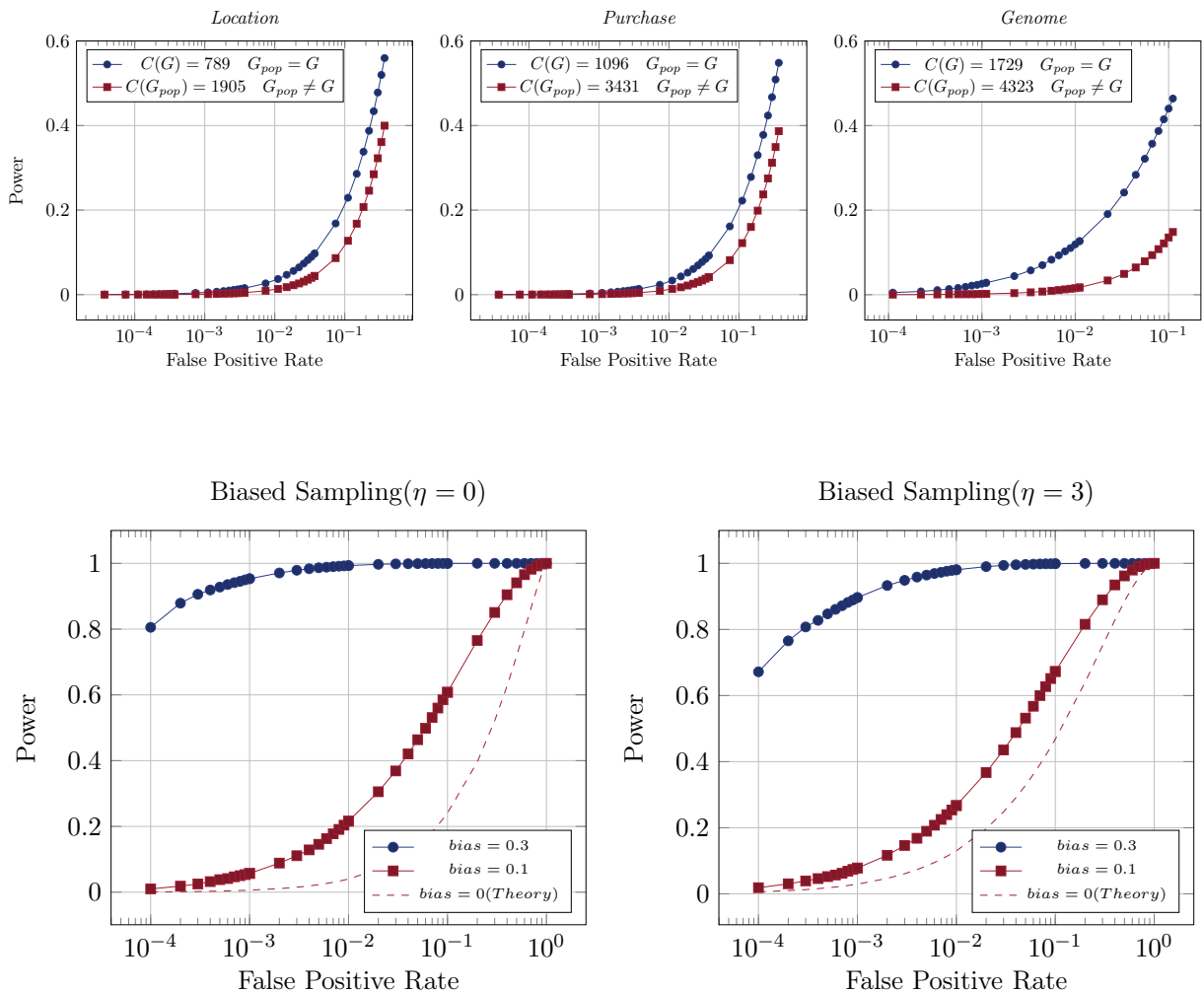


\begin{thebibliography}{10}

\providecommand{\natexlab}[1]{#1}
\providecommand{\url}[1]{\texttt{#1}}
\expandafter\ifx\csname urlstyle\endcsname\relax
  \providecommand{\doi}[1]{doi: #1}\else
  \providecommand{\doi}{doi: \begingroup \urlstyle{rm}\Url}\fi

\bibitem{agrahari2018applications}
R.~Agrahari, A.~Foroushani, T.~R. Docking, L.~Chang, G.~Duns, M.~Hudoba,
  A.~Karsan, and H.~Zare.
\newblock Applications of bayesian network models in predicting types of
  hematological malignancies.
\newblock {\em Scientific reports}, 8(1):6951, 2018.

\bibitem{backes2016membership}
M.~Backes, P.~Berrang, M.~Humbert, and P.~Manoharan.
\newblock Membership privacy in microrna-based studies.
\newblock In {\em Proceedings of the 2016 ACM SIGSAC Conference on Computer and
  Communications Security}, pages 319--330. ACM, 2016.

\bibitem{berry1941accuracy}
A.~C. Berry.
\newblock The accuracy of the gaussian approximation to the sum of independent
  variates.
\newblock {\em Transactions of the american mathematical society},
  49(1):122--136, 1941.

\bibitem{bindschaedler2017plausible}
V.~Bindschaedler, R.~Shokri, and C.~A. Gunter.
\newblock Plausible deniability for privacy-preserving data synthesis.
\newblock {\em Proceedings of the VLDB Endowment}, 10(5):481--492, 2017.

\bibitem{dong2019gaussian}
J.~Dong, A.~Roth, and W.~J. Su.
\newblock Gaussian differential privacy.
\newblock {\em arXiv preprint arXiv:1905.02383}, 2019.

\bibitem{dwork2006calibrating}
C.~Dwork, F.~McSherry, K.~Nissim, and A.~Smith.
\newblock Calibrating noise to sensitivity in private data analysis.
\newblock In {\em Theory of cryptography conference}, pages 265--284. Springer,
  2006.

\bibitem{dwork2017exposed}
C.~Dwork, A.~Smith, T.~Steinke, and J.~Ullman.
\newblock Exposed! a survey of attacks on private data.
\newblock {\em Annual Review of Statistics and Its Application}, 4:61--84,
  2017.

\bibitem{dwork2015robust}
C.~Dwork, A.~Smith, T.~Steinke, J.~Ullman, and S.~Vadhan.
\newblock Robust traceability from trace amounts.
\newblock In {\em Foundations of Computer Science (FOCS), 2015 IEEE 56th Annual
  Symposium on}, pages 650--669. IEEE, 2015.

\bibitem{esseen1942liapunoff}
C.-G. Esseen.
\newblock On the liapunoff limit of error in the theory of probability.
\newblock {\em Arkiv f{\"o}r matematik, astronomi och fysik}, 1942.

\bibitem{hall1999correlation}
M.~A. Hall.
\newblock Correlation-based feature selection for machine learning.
\newblock 1999.

\bibitem{hayes2018logan}
J.~Hayes, L.~Melis, G.~Danezis, and E.~De~Cristofaro.
\newblock Logan: Membership inference attacks against generative models.
\newblock In {\em Proceedings on Privacy Enhancing Technologies (PoPETs)},
  number~1. De Gruyter, 2018.

\bibitem{homer2008resolving}
N.~Homer, S.~Szelinger, M.~Redman, D.~Duggan, W.~Tembe, J.~Muehling, J.~V.
  Pearson, D.~A. Stephan, S.~F. Nelson, and D.~W. Craig.
\newblock Resolving individuals contributing trace amounts of dna to highly
  complex mixtures using high-density snp genotyping microarrays.
\newblock {\em PLoS genetics}, 4(8):e1000167, 2008.

\bibitem{im2012sharing}
H.~K. Im, E.~R. Gamazon, D.~L. Nicolae, and N.~J. Cox.
\newblock On sharing quantitative trait gwas results in an era of
  multiple-omics data and the limits of genomic privacy.
\newblock {\em The American Journal of Human Genetics}, 90(4):591--598, 2012.

\bibitem{jaynes1957information1}
E.~T. Jaynes.
\newblock Information theory and statistical mechanics.
\newblock {\em Physical review}, 106(4):620, 1957.

\bibitem{jaynes1957information2}
E.~T. Jaynes.
\newblock Information theory and statistical mechanics. ii.
\newblock {\em Physical review}, 108(2):171, 1957.

\bibitem{koller2009probabilistic}
D.~Koller, N.~Friedman, and F.~Bach.
\newblock {\em Probabilistic graphical models: principles and techniques}.
\newblock MIT press, 2009.

\bibitem{melis2018exploiting}
L.~Melis, C.~Song, E.~De~Cristofaro, and V.~Shmatikov.
\newblock Exploiting unintended feature leakage in collaborative learning.
\newblock {\em arXiv preprint arXiv:1805.04049}, 2018.

\bibitem{nasr2019comprehensive}
M.~Nasr, R.~Shokri, and A.~Houmansadr.
\newblock Comprehensive privacy analysis of deep learning: Passive and active
  white-box inference attacks against centralized and federated learning.
\newblock In {\em IEEE Symposium on Security and Privacy (SP)}, pages
  1022--1036, 2019.

\bibitem{neyman1933ix}
J.~Neyman and E.~S. Pearson.
\newblock Ix. on the problem of the most efficient tests of statistical
  hypotheses.
\newblock {\em Philosophical Transactions of the Royal Society of London.
  Series A, Containing Papers of a Mathematical or Physical Character},
  231(694-706):289--337, 1933.

\bibitem{pyrgelis2017knock}
A.~Pyrgelis, C.~Troncoso, and E.~De~Cristofaro.
\newblock Knock knock, who's there? membership inference on aggregate location
  data.
\newblock {\em arXiv preprint arXiv:1708.06145}, 2017.

\bibitem{sablayrolles2019white}
A.~Sablayrolles, M.~Douze, Y.~Ollivier, C.~Schmid, and H.~J{\'e}gou.
\newblock White-box vs black-box: Bayes optimal strategies for membership
  inference.
\newblock {\em arXiv preprint arXiv:1908.11229}, 2019.

\bibitem{salem2018ml}
A.~Salem, Y.~Zhang, M.~Humbert, P.~Berrang, M.~Fritz, and M.~Backes.
\newblock Ml-leaks: Model and data independent membership inference attacks and
  defenses on machine learning models.
\newblock {\em arXiv preprint arXiv:1806.01246}, 2018.

\bibitem{sankararaman2009genomic}
S.~Sankararaman, G.~Obozinski, M.~I. Jordan, and E.~Halperin.
\newblock Genomic privacy and limits of individual detection in a pool.
\newblock {\em Nature genetics}, 41(9):965, 2009.

\bibitem{shokri2017membership}
R.~Shokri, M.~Stronati, C.~Song, and V.~Shmatikov.
\newblock Membership inference attacks against machine learning models.
\newblock In {\em Security and Privacy (SP), 2017 IEEE Symposium on}, pages
  3--18. IEEE, 2017.

\bibitem{shringarpure2015privacy}
S.~S. Shringarpure and C.~D. Bustamante.
\newblock Privacy risks from genomic data-sharing beacons.
\newblock {\em The American Journal of Human Genetics}, 97(5):631--646, 2015.

\bibitem{su2013using}
C.~Su, A.~Andrew, M.~R. Karagas, and M.~E. Borsuk.
\newblock Using bayesian networks to discover relations between genes,
  environment, and disease.
\newblock {\em BioData mining}, 6(1):6, 2013.

\bibitem{wang2009learning}
R.~Wang, Y.~F. Li, X.~Wang, H.~Tang, and X.~Zhou.
\newblock Learning your identity and disease from research papers: information
  leaks in genome wide association study.
\newblock In {\em Proceedings of the 16th ACM conference on Computer and
  communications security}, pages 534--544. ACM, 2009.

\bibitem{yeom2018privacy}
S.~Yeom, I.~Giacomelli, M.~Fredrikson, and S.~Jha.
\newblock Privacy risk in machine learning: Analyzing the connection to
  overfitting.
\newblock In {\em 2018 IEEE 31st Computer Security Foundations Symposium
  (CSF)}, pages 268--282. IEEE, 2018.

\bibitem{zhang2017privbayes}
J.~Zhang, G.~Cormode, C.~M. Procopiuc, D.~Srivastava, and X.~Xiao.
\newblock Privbayes: Private data release via bayesian networks.
\newblock {\em ACM Transactions on Database Systems (TODS)}, 42(4):25, 2017.

\end{thebibliography}
\end{document}